\documentclass[11pt,a4paper,final]{article}
\pdfoutput=1
\usepackage[utf8]{inputenc}
\usepackage{amsmath}
\usepackage{graphicx}
\usepackage{xcolor}
\usepackage{eucal}
\usepackage{xfrac}
\usepackage[left=1in,right=1in,top=1in,bottom=1in]{geometry}
\usepackage{amsthm}

\usepackage{XCharter}
\usepackage[xcharter,vvarbb]{newtxmath}

\usepackage{booktabs} 
\usepackage[ruled]{algorithm2e} 
\usepackage{tikz}
\usetikzlibrary{arrows}
\usetikzlibrary{shapes.multipart}
\usepackage{caption}

\SetAlFnt{\small}
\SetAlCapFnt{\small}
\SetAlCapNameFnt{\small}
\SetAlCapHSkip{0pt}
\IncMargin{-\parindent}
\usepackage[numbers,sort&compress]{natbib}

\usepackage[T1]{fontenc}    
\usepackage{hyperref}       
\newcommand{\declarecolor}[2]{\definecolor{#1}{RGB}{#2}\expandafter\newcommand\csname #1\endcsname[1]{\textcolor{#1}{##1}}}

\declarecolor{White}{255, 255, 255}
\declarecolor{Black}{0, 0, 0}
\declarecolor{Maroon}{128, 0, 0}
\declarecolor{Coral}{255, 127, 80}
\declarecolor{Red}{182, 21, 21}
\declarecolor{LimeGreen}{50, 205, 50}
\declarecolor{DarkGreen}{0, 90, 0}
\declarecolor{Navy}{0, 0, 128}

\definecolor{plotblue}{HTML}{377eb8}
\definecolor{plotorange}{HTML}{ff7f00}
\definecolor{plotgreen}{HTML}{4daf4a}

\hypersetup{
    colorlinks=true,
    citecolor=DarkGreen,
    linkcolor=Navy,
    filecolor=magenta,      
    urlcolor=black,
    pdfpagemode=FullScreen,
}
\usepackage{url}            
\usepackage{booktabs}       
\usepackage{nicefrac}       
\usepackage{microtype}      
\usepackage{mathtools}
\usepackage{bbm}
\usepackage{wrapfig}
\usepackage{floatrow}
\usepackage{enumitem}
\usepackage{authblk}
\usepackage{cleveref}
\usepackage{comment}
\usepackage{parskip}

\usepackage{color-edits}
\addauthor{kh}{red}

\usepackage{caption}
\usepackage{subcaption}
\usepackage{thm-restate}

\newtheoremstyle{spaced}
  {0.5em}          
  {0em}          
  {\normalfont}  
  {}             
  {\bfseries}    
  {.}            
  { }            
  {}             

\theoremstyle{spaced}

\newtheorem{theorem}{Theorem}[section]
\newtheorem{lemma}[theorem]{Lemma}

\newtheorem{definition}[theorem]{Definition}

\newcommand{\E}{\mathbb{E}}

\newcommand{\cA}{\mathcal{A}}

\newcommand{\cD}{\mathcal{D}}
\newcommand{\cS}{\mathcal{S}}
\newcommand{\cM}{\mathcal{M}}
\newcommand{\cL}{\mathcal{L}}
\newcommand{\cC}{\mathcal{C}}
\newcommand{\cP}{\mathcal{P}}
\newcommand{\cR}{\mathcal{R}}
\newcommand{\su}{\mathsf{u}}
\newcommand{\sT}{\mathsf{T}}
\newcommand{\sfc}{\mathsf{c}}
\newcommand{\sd}{\mathsf{d}}
\newcommand{\vb}{\mathbf{b}}
\newcommand{\vy}{\mathbf{y}}
\newcommand{\vc}{\mathbf{c}}
\newcommand{\boldv}{\mathbf{v}}

\newcommand{\va}{\mathbf{a}}

\newcommand{\vx}{\mathbf{x}}
\newcommand{\vp}{\mathbf{p}}
\newcommand{\opt}{\mathrm{OPT}}
\newcommand{\vcg}{\mathtt{VCG}}
\newcommand{\child}{\mathsf{c}}

\begin{document}

\title{\textbf{Inference Auctions}}
\author[1]{Keegan Harris\thanks{,$^\dagger$: Denotes equal contribution.}}
\author[2]{Siddharth Prasad\protect\footnotemark[1]}
\author[3]{Asher Trockman\protect\footnotemark[1]}
\author[1]{\\Nika Haghtalab$^\dagger$}
\author[1,4]{Michael I. Jordan$^\dagger$}

\affil[1]{University of California, Berkeley}
\affil[2]{Toyota Technological Institute at Chicago}
\affil[3]{Google Research}
\affil[4]{Inria \& \'Ecole Normale Sup\'erieure}
\affil[ ]{\texttt{\{keegan.harris,nika,michael\_jordan\}@berkeley.edu}}
\affil[ ]{\texttt{sprasad@ttic.edu}, \texttt{trockman@google.com}}

\date{}

\maketitle

\begin{abstract}
    When inference demand exceeds available compute capacity, model providers must decide which requests should be served first. 
Users have different tolerances for delay from an LLM API, but current priority pricing schemes compress these differences into coarse fixed-price service tiers.
We design an \emph{inference auction} that allows users to bid for faster service. Our auction allocates priority in an economically efficient way without sacrificing latency, and we develop fast algorithms for implementing prices that incentivize truthful bidding.
We also design an autobidding agent for our inference auction, where users specify an inference budget and the autobidder dynamically adjusts its bids over time to maximize user utility subject to the budget constraint. 
Experiments validate the practicality of our auction: it increases system welfare while maintaining the cache utilization and latency advantages of SGLang, a state-of-the-art inference serving framework.\looseness-1 
\end{abstract}

\section{Introduction}
Large-scale LLM inference requires allocation of shared compute resources across different users' requests. 
During periods of heavy traffic, model providers need to determine which requests to serve first, and how much to charge for faster service. 
Users can differ substantially in their willingness to pay for earlier service. For example, an interactive request may be sensitive to delays that are inconsequential for a background job. 
Today, most major model providers allocate priority through a small number of fixed-price service tiers (typically two or three), with higher-paying requests receiving priority over lower-paying ones during times of high traffic~\citep[e.g.,][]{anthropic_service_tiers, aws_bedrock_service_tiers, openai_priority_processing, google2026priority, xai2026priority}. 
These coarse tiers give users limited scope to express their latency preferences, 
and as a result, priority is allocated in an economically inefficient way.  

The design of a more expressive and efficient market for inference priority raises new challenges that are absent from conventional queuing systems and transaction-fee markets.
%
Modern inference engines process prompts with the same prefix by reusing the key-value (KV) cache across their shared prefix, so the time it takes to serve a request depends on which requests have already been served~\citep[e.g.,][]{zheng2024sglang,kwon2023efficient, qin2025mooncake}. 
For example, interrupting a batch of requests that all share a long common prefix to serve unrelated requests may cause the common prefix to be evicted from the KV cache, disrupting cache reuse and harming request throughput. Serving that same batch of common-prefix requests contiguously instead will generally yield improved throughput due to better cache utilization.
%
This form of {\em prefix-aware} scheduling and KV cache management is implemented in state-of-the-art inference serving frameworks such as SGLang~\citep{zheng2024sglang}, leading to significant latency and throughput gains. 
A good market design should not sacrifice these performance gains.\looseness-1


%
We design an {\em inference auction} that allows users to bid for priority. 
Our key observation is that preserving the cache-reuse benefits of prefix-aware serving systems still leaves freedom to decide which groups of requests receive priority, enabling significant gains in social welfare while retaining the cache optimality of practical serving systems. 
Our key contributions, in more detail, are as follows:\looseness-1

%
\textbf{Inference Auction Design (Section~\ref{sec:env} and Section~\ref{sec:dfs}).}
We develop the first inference auction, a way for LLM providers to prioritize requests by letting users bid for faster service. 
Each bid represents the user’s reported value for moving a request from the back of the scheduling queue to the front, with value increasing linearly as the request moves earlier in the schedule. Our goal is to maximize social welfare---the total value users derive from their assigned priorities---while retaining cache optimality. 
A key aspect of our inference auction design is to constrain the resulting schedule of requests to be a {\em depth-first-search (DFS) schedule}, that is, it must correspond to a DFS traversal of the \emph{radix tree} of requests (a tree in which requests that share a prompt prefix share a path from the root). Inference auctions corresponding to such DFS schedules maximize KV cache hit rate due to optimal prefix sharing.
We show that the welfare-maximizing DFS schedule corresponds to the DFS traversal that orders a node's children by average bid across requests in their subtrees, and can therefore be computed very efficiently. 
%
%
We pair our scheduling rule with Vickrey–Clarke–Groves (VCG) payments that incentivize truthful bidding.
%
Computing VCG payments naively requires recomputing an optimal schedule for each user, causing a quadratic worst-case dependence on the number of requests. We develop an algorithm that computes all payments in quasilinear run-time by decomposing each user’s welfare externality into local changes along the radix tree.
\looseness-1

\textbf{Autobidding Agents for Inference Auctions (Section~\ref{sec:autobidding}).} 
Inference users typically submit many requests over time and must manage their inference budget. 
Even when truthful reporting is optimal in each auction in isolation, users still need to choose how aggressively to bid across requests so that their inference budget is spent where it is most valuable. 
%
To address this, we design an autobidding agent that adjusts a user's bids over time to maximize their cumulative utility, subject to a budget constraint. 
Our autobidding algorithm scales each request's value by a multiplier and adjusts this multiplier over time using gradient descent. 
We show that this strategy achieves $O(\sqrt{T})$ regret relative to the optimal-in-hindsight bidding policy on a sequence of $T$ inference auctions. 
We also characterize the welfare loss when users have different pacing multipliers:  
we show that welfare only decreases when these differences reverse the ordering of requests in the DFS schedule, and we bound the resulting inefficiency by the variation in pacing multipliers across the radix subtrees.

\textbf{Experiments (Section~\ref{sec:expts}).} 
We implement our inference auction as a scheduling layer in front of an unmodified SGLang server and evaluate it across several workloads with various shared prefix structures. 
Compared to a bid-blind, cache-optimal request ordering, our mechanism improves overall welfare by 12--18.6\% under i.i.d. request valuations and up to $100\%$ when valuations are non-i.i.d., all while leaving cache hit rate and average time-to-first-token virtually unchanged. 
Simultaneously, we find that our mechanism enjoys $\mathord{\sim}80\%$ of the welfare of the unconstrained schedule ordered by decreasing bids. 
In contrast, this unconstrained schedule sharply reduces cache hit rates, which results in an increase in average latency by as much as $12\times$.\looseness-1 
\section{Cache-Optimal Schedules and Our Bidding Environment}\label{sec:env}

Serving an autoregressive language model has two stages. 
During \emph{prefill}, the model processes a request's prompt and stores its keys and values in the KV cache. 
During \emph{decoding}, it generates output tokens while reusing the cached keys and values. 
Modern inference engines interleave the two stages across requests, typically prioritizing the decode steps of requests already in progress. 
Our mechanism determines the order in which queued requests are admitted for prefill. 

The KV cache entries for a prompt prefix can be reused by other requests which share that prefix, reducing prefill computation. 
Because cache capacity is limited, this reuse depends on the request scheduling order: processing requests with a long shared prefix consecutively can avoid evicting and recomputing their common entries. 
The inference engine SGLang~\citep{zheng2024sglang} puts this principle into practice by storing shared token prefixes in a radix tree and scheduling requests in the order of a DFS traversal of that tree, which is a longest-prefix-match order and is cache-optimal. 
{\em A key aspect of our auction design is to allow bids to only affect the ordering of children in a DFS traversal of the request radix tree.} This ensures that no matter the precise economics of the auction (e.g., schedules, payments, incentives), we maintain the cache efficiency of prefix-aware scheduling. We now formally introduce the elements of the auction environment.\looseness-1

\textbf{Notation.} 
%
$\cS_n$ denotes the set of all permutations of $\{1, \ldots, n\}$.
For a set of indexed elements $\mathbf{a}_{1:m} \coloneq \{\mathbf{a}_1, \ldots, \mathbf{a}_m\}$, we write $\mathbf{a}_{-i} \coloneq \mathbf{a}_{1:m} \setminus \{\mathbf{a}_i\}$ to denote $\mathbf{a}_{1:m}$ with element $i$ excluded.\looseness-1 

\textbf{Requests and Scheduling Batches.}
We consider a setting where a model provider is serving requests to a single LLM on one inference worker. 
Requests arrive asynchronously and are periodically grouped into scheduling batches of $n$ queued requests. 
We treat the formation of these batches as exogenous: our mechanism determines the relative prefill priority of requests within a batch, but not which requests enter the batch.\footnote{In practice, batches would be formed only during times of high traffic when requests start to queue.}\looseness-1 

Let $\cR \coloneq \{1, \ldots, n\}$ be the requests in a batch. 
There are $m \leq n$ inference users who may each submit multiple requests. 
We use $\cR_i \subseteq \cR$ to denote the set of requests submitted by user $i$, where $\bigcup_{i=1}^m \cR_i = \cR$ and $\cR_i\cap\cR_j=\emptyset$ for all $i \neq j$. 
A \emph{schedule} is a permutation $\sigma \in \cS_n$, where $\sigma(r)$ is the position of request $r$ within the batch. 
Smaller positions receive earlier prefill service, and the scheduler admits requests for prefill according to this ordering. 
Each request is represented by its sequence of prompt tokens. 
Let $M$ denote the total number of prompt tokens across all requests in $\cR$, counting tokens separately for each request. 
The shared-prefix structure of $\cR$ is captured by a radix tree. 

\begin{definition}[Request Radix Tree]
The radix tree for a set of requests is the rooted, edge-labeled tree obtained by merging common prefixes and contracting every non-branching path into a single edge. 
Each request is represented by an auxiliary leaf node (that stores request metadata, e.g., user ID and user bid), and the lowest common ancestor of two request leaves represents their longest shared prefix. $\sT$ denotes the request-augmented radix tree.
For a node $\su$, let $\cC(\su)$ denote its children, $\cL(\su)$ denote the request leaves in the subtree rooted at $\su$, $d(\su)\coloneq |\cC(\su)|$, and $N(\su) \coloneq |\cL(\su)|$. 
The maximum branching factor of $\sT$ is $\nu \coloneq \max_{\su} |\cC(\su)|$. 
See Figure~\ref{fig:radix-tree} for an example.\looseness-1
\end{definition}  

%

Of particular interest are schedules that process all requests in one radix subtree before moving on to another, that is, they correspond to DFS traversals of $\sT$.

\begin{definition}[DFS Schedule]
    A schedule $\sigma$ is a DFS schedule if for every node $\su$, the leaves belonging to each child subtree $\sfc \in \cC(\su)$ appear contiguously in $\sigma$. 
    Let $\cD(\sT) \subseteq \cS_n$ denote the set of DFS schedules for $\sT$. 
\end{definition}
%
%
Under the assumption that the KV cache is large enough to hold the longest request in the batch,
schedules in $\cD(\sT)$ enjoy the optimal cache hit rate by Theorem 3.1 in~\citet{zheng2024sglang}.\looseness-1 

\textbf{Values and Priority.} 
The user who submits request $r$ has a private value $v_r \geq 0$ for receiving earlier service. 
The \emph{fractional priority} received by request $r$ under schedule $\sigma$ is $x_r(\sigma) \coloneq \frac{|s \in \cR : \sigma(s) > \sigma(r)|}{n-1}$. 
The first request has fractional priority $1$, the last has fractional priority $0$, and a request scheduled ahead of an $\alpha$-fraction of the batch receives priority $\alpha$. 
Let $\boldv_i \coloneq [v_r]_{r \in \cR_i}$ denote the vector of user $i$'s values for his requests and $\vx_i(\sigma) \coloneq [x_r(\sigma)]_{r \in \cR_i}$ be the vector of the corresponding fractional priorities under $\sigma$. 
User $i$'s total value under $\sigma$, interpreted as his maximum willingness-to-pay for fractional priority $\vx_i(\sigma)$, is $V_i(\sigma) \coloneq \langle \boldv_i, \vx_i(\sigma) \rangle$ and the \emph{welfare} of $\sigma$ is $W(\sigma; \boldv_{1:m}) \coloneq \sum_{i=1}^m V_i(\sigma)$. 

Fractional priority provides a simple rank-based proxy for queueing delay and the resulting time-to-first-token (TTFT). 
{\em Modeling value as being linear in fractional priority is a design choice to keep the bidding language simple}: users only need to report how much they value moving a request from the back of the scheduling batch to the front, rather than specify a potentially complex value function over TTFT. 
Our experiments show that bidding according to fractional priority leads to schedules that closely track realized TTFT, supporting fractional priority as a practical proxy for latency.\looseness-1

\textbf{Bids, Schedules, and Payments.} 
Each request $r$ has a corresponding user bid $b_r \geq 0$, interpreted as the user's reported value for $r$ receiving full priority. 
Let $\vb_i \coloneq [b_r]_{r \in \cR_i}$ be the vector of user $i$'s bids. 
An inference auction is a mechanism which maps bids to a DFS schedule and user payments. 

\begin{definition}[Inference Auction]
    Given a radix tree $\sT$, an inference auction consists of (1) a \emph{scheduling rule} $A : \mathbb{R}_+^n \rightarrow \cD(\sT)$ mapping bids to DFS schedules and (2) a \emph{payment rule} $p : \mathbb{R}_+^n \rightarrow \mathbb{R}_+^m$ mapping bids to payments, where $p_i(\vb_{1:m})$ is the payment charged to user $i$. 

    {\em Any inference auction enjoys optimal cache utilization by Theorem 3.1 of~\citet{zheng2024sglang}.}
\end{definition}

We make the standard assumption that users have quasilinear utility, so a user with true values $\boldv_i$ obtains utility $V_i(A(\vb_{1:m})) - p_i(\vb_{1:m})$ in an inference auction with scheduling rule $A$ and payment rule $p$. 
Our goal in the sequel is to design an inference auction that maximizes welfare over $\cD(\sT)$ while incentivizing truthful bidding. 
A simple symmetry argument shows that the welfare loss due to restricting to $\cD(\sT)$ is no more than $50\%$.\footnote{At the same time, the welfare-maximizing DFS schedule can improve significantly over an arbitrary schedule in $\cD(\sT)$. 
As an example, consider the setting where one request has value $v^\star > 0$ and shares no common prefix with the $n-1$ other requests, each of which has value $0$ and a common shared prefix. 
A uniformly random DFS schedule would obtain zero welfare with probability $1/2$, whereas the optimal DFS schedule always achieves the maximum welfare $v^\star$.}

\begin{restatable}{proposition}{propwelfare}\label{prop:dfs-welfare-loss}
    For any $\boldv_{1:m}$, $\max_{\sigma\in\cD(\sT)}W(\sigma;\boldv_{1:m})\ge\frac{1}{2}\max_{\sigma\in\cS_n}W(\sigma;\boldv_{1:m})$.
\end{restatable}

\section{Welfare-Optimal Inference Auctions}\label{sec:dfs}

We now design a welfare-maximizing (or {\em economically efficient}) inference auction that is simple and tractable. 
In Section~\ref{sec:welfare}, we design a scheduling rule that maximizes reported user welfare. 
The scheduling rule admits a clean algorithm: at every node of the radix tree, visit its children in decreasing order of their average sub-tree bid. 
In Section~\ref{sec:VCG}, we implement VCG prices that incentivize truthful bidding. 
Although these payments are defined through a separate counterfactual optimization for each user, we show that they can be computed in quasilinear total run-time. 
Taken together, our scheduling and payment rules maximize \emph{true} user welfare and can be implemented with very little overhead compared to the time it takes to construct and traverse the radix tree.\footnote{Inference frameworks such as SGLang already construct and traverse the radix tree to implement prefix-aware scheduling. We only incur an additional $O(\log \nu)$ multiplicative factor.\looseness-1} 

\subsection{Implementing Welfare-Maximizing DFS Schedules}\label{sec:welfare}

Since we restrict to DFS schedules, the scheduler's only decision is how to order the children of each internal node in the radix tree. 
For a node $\su$, let $F_b(\su) \coloneq \sum_{r \in \cL(\su)} b_r$ denote the sum of bids in the subtree rooted at $\su$. 
We will show that reported welfare is maximized by visiting the children of every internal node in decreasing order of $F_b(\su) / N(\su)$. 
%

For intuition as to why this is the case, consider two children $\sfc$ and $\sd$ of the same internal node. 
If $\sfc$ is visited before $\sd$, every request in $\cL(\sfc)$ is scheduled ahead of every request in $\cL(\sd)$, and these cross-subtree pairs contribute $\frac{N(\sd) F_b(\sfc)}{n-1}$ to reported welfare. 
Visiting $\sd$ before $\sfc$ instead contributes $\frac{N(\sfc) F_b(\sd)}{n-1}$. 
It is therefore welfare-improving to place $\sfc$ before $\sd$ whenever $N(\sd) F_b(\sfc) \geq N(\sfc) F_b(\sd)$, or equivalently when $F_b(\sfc) / N(\sfc) \geq F_b(\sd) / N(\sd)$. 
The following theorem shows that applying this local rule independently at every branch produces a globally-optimal DFS schedule. 

\begin{figure}[t]
\centering
\begin{minipage}[t]{0.40\linewidth}
\vspace{0pt}
\begin{algorithm}[H]
\SetAlgoLined
\SetAlgoNoEnd
Let $\su$ be the root of the radix tree $\sT$\\
Initialize $L\gets[\,]$\\
$\mathrm{DFS}(\su,L)$\\
\textbf{return} $L$

\caption{Average-Bid DFS Scheduling Rule $\Bar{A}(\vb_{1:m})$}
\label{alg:auction}
\end{algorithm}
\end{minipage}%
\hfill
\begin{minipage}[t]{0.58\linewidth}
\vspace{0pt}
\begin{algorithm}[H]
\SetAlgoLined
\SetAlgoNoEnd

\If{$\su$ \textup{is a leaf}}{
    $L\mathtt{.append}(\su)$
}

$[\mathsf{c}_1,\ldots,\mathsf{c}_{d(\su)}]
\gets\cC(\su)$ sorted in descending order of $F_b/N$

\For{$k=1,\ldots,d(\su)$}{
    $\mathrm{DFS}(\mathsf{c}_k,L)$
}

\caption{$\mathrm{DFS}(\su,L)$}
\label{alg:dfs}
\end{algorithm}
\end{minipage}
\end{figure}
\begin{restatable}{theorem}{thmwelfare}\label{thm:welfare}
    Given bid vectors $\vb_{1:m}$, the average-bid DFS scheduling rule $\Bar{A}(\vb_{1:m})$ given by Algorithm~\ref{alg:auction} returns a cache-optimal schedule $\sigma$ that maximizes $W(\sigma;\vb_{1:m})$ among all DFS schedules. Furthermore, this schedule can be computed in $O(M + n\log \nu)$ expected run-time. 
\end{restatable}
\noindent Figure~\ref{fig:avg-bid-dfs} illustrates an example of an average-bid DFS schedule. 

\subsection{Computing VCG Payments}\label{sec:VCG}

The average-bid DFS of Algorithm~\ref{alg:auction} yields a schedule that maximizes reported welfare, but our ultimate goal is to maximize welfare with respect to users' true values. 
%
%
\citet{Vickrey61:Counterspeculation}-\citet{Clarke71:Multipart}-\citet{Groves73:Incentives} (VCG) payments
ensure this property  
%
by charging each user for the welfare loss that his bids impose on the other users.  
When paired with a scheduling rule that maximizes reported welfare over a fixed set of feasible outcomes, these payments make truthful bidding (i.e., bidding $\vb_i = \boldv_i$) an optimal strategy for users with quasilinear utility. 
We instantiate the VCG mechanism with average-bid DFS schedules and $\cD(\sT)$ as the feasible set: 

\begin{definition}[VCG payment rule]
    Let $h_i(\vb_{-i}) = \max_{\sigma\in\cD(\sT)} W(\sigma; (\mathbf{0}, \vb_{-i}))$ denote the maximum reported welfare the other users could obtain if user $i$'s bids were set to zero while its requests remain in the radix tree.\footnote{The VCG counterfactual $h_i(\vb_{-i})$ is often described informally as removing user $i$ from the auction. In our setting, removing a user's requests would change the radix tree and hence the set of feasible schedules. In fact, this (incorrect) version of VCG would require the platform to pay some users (see Appendix~\ref{app:dfs}). We instead fix the set of requests and assign user $i$ his weakest possible type, the zero bid vector.} 
    The VCG payment rule charges user $i$ 
    \begin{equation*}
        p_i^{\vcg}(\vb_{1:m}) = h_i(\vb_{-i}) - \sum_{j\neq i} \langle \vb_{j}, \vx_{j}(\Bar{A}(\vb_{1:m}))\rangle.  
    \end{equation*}
\end{definition}
The inference auction with average-bid DFS schedules and VCG payments satisfies the following properties:\looseness-1 
\begin{enumerate}
    \item \textbf{Efficiency:} Under truthful bidding, true welfare $W(\sigma; \boldv_{1:m})$ is maximized over DFS schedules.\looseness-1 
    \item \textbf{Incentive compatibility (IC):} $V_i(\Bar{A}(\boldv_i, \vb_{-i})) - p_i^{\vcg}(\boldv_i, \vb_{-i}) \geq V_i(\Bar{A}(\vb_i, \vb_{-i})) - p_i^{\vcg}(\vb_i,\vb_{-i})$ for all $\boldv_i, \vb_i, \vb_{-i}$. This makes truthful bidding an optimal strategy. 
    \item \textbf{Individual rationality (IR):} $V_i(\Bar{A}(\boldv_i, \vb_{-i})) - p_i^{\vcg}(\boldv_i, \vb_{-i}) \geq 0$ for all $\boldv_i,\vb_{-i}$.
    \item \textbf{No positive transfers:} $p_i^{\vcg}(\vb_{1:m}) \geq 0$ for all bid vectors $\vb_{1:m}$. 
\end{enumerate}
Furthermore by~\citet{krishna1998efficient}, these payments maximize the model provider's revenue among all IC and IR mechanisms that implement the welfare-maximizing DFS schedule.\looseness-1 

\textbf{Efficient Payment Computation.} 
Naively computing VCG payments for each user consists of repeatedly zeroing out that user's bids and rerunning Algorithm~\ref{alg:auction}. 
This approach takes $O(mn\log\nu)$ time, which is quadratic in $n$ if each user has only a single request. 
While such a quadratic dependence on $n$ is unavoidable if the goal is to return the welfare-optimal DFS traversal when each user is excluded, we are able to bypass this barrier since we only need the optimal welfare values. 
As a result, our algorithm enjoys an improved runtime of $O(\min\{M, mn\} \log \nu)$.\looseness-1

The key idea behind our algorithm (Appendix~\ref{app:vcg_alg}) is easiest to see when each user submits a single request. 
%
Suppose we set the bid of request $r$ to zero. 
At every internal node on the path from the root to $r$'s leaf, only the average bid of the child containing $r$ changes. 
This average can only decrease, and so the affected child can only move later in the welfare-optimal traversal while all other children retain their relative order. 
Siblings that remain on the same side of the affected child do not contribute to the user's payment, while each sibling that overtakes it contributes a local externality. 
Since the children are already sorted by average bid from the average-bid DFS schedule, a binary search identifies the siblings that overtake the affected child. 
Summing these local effects via pre-computed prefix sums over the root-to-leaf path then gives the user's VCG payment.\looseness-1 

When a user submits multiple requests, zeroing out all of his bids may affect multiple children of a given internal node. 
However the same principle still applies: jointly reorder the affected children according to their modified average bids and merge them with the unchanged children. 
\begin{restatable}{theorem}{thmprices}\label{thm:payments}
    Given as input the processed radix tree $\sT$, where at every node $\su$ (1) $F_b(\su)$ and $N(\su)$ are stored and (2) its children are sorted in decreasing order of average bid, all VCG payments $p_1^{\vcg}(\vb_{1:m}),\ldots, p_m^{\vcg}(\vb_{1:m})$ can be computed in $O(\min\{M,mn\}\log \nu)$ total run-time.
\end{restatable}

The end-to-end run-time of implementing the welfare-optimal DFS schedule with VCG prices is therefore $O(M+ \min\{M, mn\}\log\nu)$.
\section{Autobidding Agents for Inference Auctions}\label{sec:autobidding}

The VCG mechanism makes truthful bidding optimal within any single scheduling batch. 
However, inference users typically have many requests that need to be served over time, and they may have a budget allocated for the total amount they can spend in a given time period. 
This budget constraint couples the user's bids, as bidding truthfully in every batch can exhaust the budget too early. 
Motivated by autobidding systems in online advertising,\footnote{See Section~\ref{sec:related} for more background in autobidding in ad auctions.} we develop an autobidding agent for inference auctions: an online policy that converts a user's values and remaining budget into bids, without requiring a model of future demand or competing bids. 

\textbf{Repeated Auction Model.} 
Consider a sequence of $T$ scheduling batches. 
In batch $t \in [T]$, let $\sT^{(t)}$ be the request radix tree, $\cR_1^{(t)}$ be the autobidder's requests, $\boldv_1^{(t)} = [v_r^{(t)}]_{r \in \mathcal{R}_1^{(t)}}$ be his value vector, $\vb_1^{(t)}$ be his bid vector, and $\vb_{2:m}^{(t)}$ be the other users' bids.\footnote{We use the convention that the autobidder's index is $1$, which is without loss of generality.} 
We consider the welfare-maximizing inference auction of Section~\ref{sec:dfs}. 
Given bids $\vb_{1:m}^{(t)}$, the mechanism selects schedule $\Bar{A}^{(t)}(\vb_{1:m}^{(t)})$ and charges the autobidder his VCG payment $p_1^{(t)}(\vb_{1:m}^{(t)})$. 
The autobidder's value in batch $t$ is $V_1^{(t)}(\sigma^{(t)}) \coloneq \langle \boldv_1^{(t)}, \vx_1^{(t)}(\sigma^{(t)}) \rangle$. 
The autobidder has total budget $B = T \gamma$ (where $\gamma > 0$ is the target spend per batch), and seeks to maximize cumulative utility subject to the budget constraint $\sum_{t=1}^T p_1^{(t)}(\vb^{(t)}_{1:m}) \leq T \gamma$. 
We measure the performance of the autobidder through the notion of regret. 

\begin{algorithm}[t]
\SetAlgoLined
\SetAlgoNoEnd
\textbf{Input:} Target spend $\gamma$, step size $\eta>0$\\
Initialize $B^{(1)} = T\gamma$, $\mu^{(1)}=0$\\
\For{$t = 1,\ldots, T$}{
    Let $\alpha^{(t)} = \frac{1}{1 + \mu^{(t)}}$ and bid $\vb_1^{(t)} = \alpha^{(t)} \boldv_1^{(t)} \mathbbm{1}\{\|\alpha^{(t)} \boldv_1^{(t)}\|_1 \leq B^{(t)}\}$\\ 
    Observe payment $p_1^{(t)}(\vb_{1:m}^{(t)})$ and update budget $B^{(t+1)} \leftarrow B^{(t)} - p_1^{(t)}(\vb_{1:m}^{(t)})$\\
    Update $\mu^{(t+1)} \leftarrow \left[\mu^{(t)}+\eta(p^{(t)}_1(\vb_{1:m}^{(t)})-\gamma)\right]_+$\\
}
\caption{Pacing Autobidder for Welfare-Optimal Inference Auctions}\label{alg:ogd}
\end{algorithm}

\begin{definition}[Autobidder Regret]\label{def:reg}
    For a realized sequence of instances $(\boldv^{(t)}_1, \sT^{(t)}, \vb_{2:m}^{(t)})_{t=1}^T \sim \cP$, the optimal cumulative utility in hindsight is
    \begin{equation*}
        \opt(T) \coloneq  \max_{\vb^{(1)}_1, \ldots \vb^{(T)}_1} \left\{ \sum_{t=1}^T V_1^{(t)}(\sigma^{(t)}) - p_1^{(t)}(\vb_{1:m}^{(t)}) \; : \; \sum_{t=1}^T p_1^{(t)}(\vb_{1:m}^{(t)}) \leq T\gamma \right\}.
    \end{equation*}
    
    Regret is the expected difference between $\opt(T)$ and the autobidder's cumulative utility; i.e., 
    \begin{equation*}
        R(T) \coloneq  \E \left[\opt(T) - \sum_{t=1}^T (V_1^{(t)}(\sigma^{(t)}) - p_1^{(t)}(\vb_{1:m})) \right].
    \end{equation*}
    %
    %
\end{definition}

\textbf{Value Pacing.} 
At every timestep, our autobidder scales the request values by a common pacing factor $\alpha^{(t)}$, bids $\vb_1^{(t)} = \alpha^{(t)} \boldv_1^{(t)}$, and updates the pacing factor $\alpha^{(t)} \rightarrow \alpha^{(t+1)}$. 
The pacing factor decreases when spending exceeds the per-batch target $\gamma$, and increases when spending falls below the target. 
Pseudocode for our autobidding agent is given in Algorithm~\ref{alg:ogd}.\looseness-1 

\begin{restatable}{theorem}{thmab}\label{thm:regret}
    Suppose that $(\boldv^{(t)}_1, \sT^{(t)}, \mathbf{b}_{2:m}^{(t)})_{t=1}^T$ are drawn i.i.d. from an unknown distribution $\mathcal{P}$. 
    Let $K := \max_{t \in [T]} |\cR_1^{(t)}|$, $\|\boldv^{(t)}_1\|_{\infty} \leq C$, and suppose that schedules and payments are assigned according to the inference auction of Section~\ref{sec:dfs}.  
    Then if $\eta = \Theta(1/\sqrt{T})$, Algorithm~\ref{alg:ogd} obtains regret $R(T) = O(\mathrm{poly}(C,K,\gamma) \sqrt{T})$ when $T$ is sufficiently large. 
\end{restatable}
The proof proceeds via a reduction to~\citet{balseiro2020dual}, who study online allocation problems with generic concave reward functions and resource constraints. 
The main technical difficulty is showing that their framework actually applies in our setting. 
The distributional assumption models a stationary serving environment in which request values, prefix structures, and competing bids can vary across batches but arise from a stable workload. 

\textbf{Welfare Under Heterogeneous Pacing.}
When multiple autobidders pace independently, their multipliers can differ, and so their bids will generally not be proportional to a common pacing factor. 
This can lead to a loss in per-batch social welfare compared to the welfare under truthful reporting. 
We now quantify this loss as a function of the request radix tree's structure. 

For the remainder of this section, we suppress the batch index and write each bid as $b_r = \alpha_r v_r$. 
For a subtree $\sfc$, let $F_v(\sfc) \coloneq \sum_{r \in \cL(\sfc)} v_r$ and define its \emph{value-weighted pacing factor} by $\alpha(\sfc) \coloneq F_b(\sfc) / F_v(\sfc)$.\looseness-1 

\begin{definition}[Node Pacing Condition Number]\label{def:subtree_pacing}
    The condition number of node $\su$ is $\kappa_{\su} \coloneq \max\{\alpha(\sfc) : \sfc \in \cC(\su), F_v(\sfc) > 0\} / \min\{\alpha(\sfc) : \sfc \in \cC(\su), F_v(\sfc) > 0\}$, with $\kappa_{\su} \coloneq  1$ if at most one child of $\su$ has $F_v(\sfc) > 0$. 
\end{definition}

Condition number measures how unevenly sibling subtrees are paced. 
When $\kappa_{\su} = 1$, pacing preserves the welfare-maximizing order, whereas larger values of $\kappa_{\su}$ result in greater local distortion. 
The following theorem formalizes this relationship by lower-bounding the realized welfare at a single timestep in terms of the node-level condition numbers. 

\begin{restatable}{theorem}{thmcondition}\label{thm:condition} 
    Suppose that each request $r$ has corresponding bid $b_r = \alpha_r v_r$ for pacing factors $\alpha_1, \ldots, \alpha_n \in (0, 1]$, and let $\Omega_{\su}$ denote the marginal contribution of node $\su$ to the optimal DFS welfare. 
    Then the social welfare for the batch is bounded as
    \begin{equation*}
        \frac{W(\Bar{A}(\vb_{1:m}); \boldv_{1:m})}{\max_{\sigma \in \cD(T)} W(\sigma; \boldv_{1:m})} \geq \frac{\sum_{\su} \Omega_{\su} / \kappa_{\su}}{\sum_{\su} \Omega_{\su}} \geq \frac{1}{\max_{\su} \kappa_{\su}}.
    \end{equation*}
\end{restatable}

Theorem~\ref{thm:condition} isolates the economically relevant source of inefficiency from value pacing in inference auctions. 
Pacing dispersion only matters when it changes the average-bid ordering of sibling subtrees, and the loss in welfare is weighted by $1/\kappa_{\su}$. 
Heterogeneous request-level pacing may therefore be harmless when it ``averages out'' within subtrees; what matters more is any systematic pacing variation across sibling subtrees that directly compete for scheduling priority. 
See Figure~\ref{fig:pacing} for an example.
\section{Experiments}\label{sec:expts}
\providecommand{\dfsbid}{\textsc{dfs-bid}}
\providecommand{\dfsuniform}{\textsc{dfs-uniform}}
\providecommand{\bidsort}{\textsc{bid-sort}}
\providecommand{\nosort}{\textsc{no-sort}}
\providecommand{\fcfs}{\textsc{fcfs}}
\providecommand{\mixshare}[1]{\textsc{mixshare}-#1}
\providecommand{\mixsharename}{\textsc{mixshare}}
\providecommand{\agent}{\textsc{agent}}
\providecommand{\tot}{\textsc{tot}}
\providecommand{\totlong}{\textsc{tot-long}}

\begin{figure}[t]
    \centering
    \includegraphics[width=\linewidth]{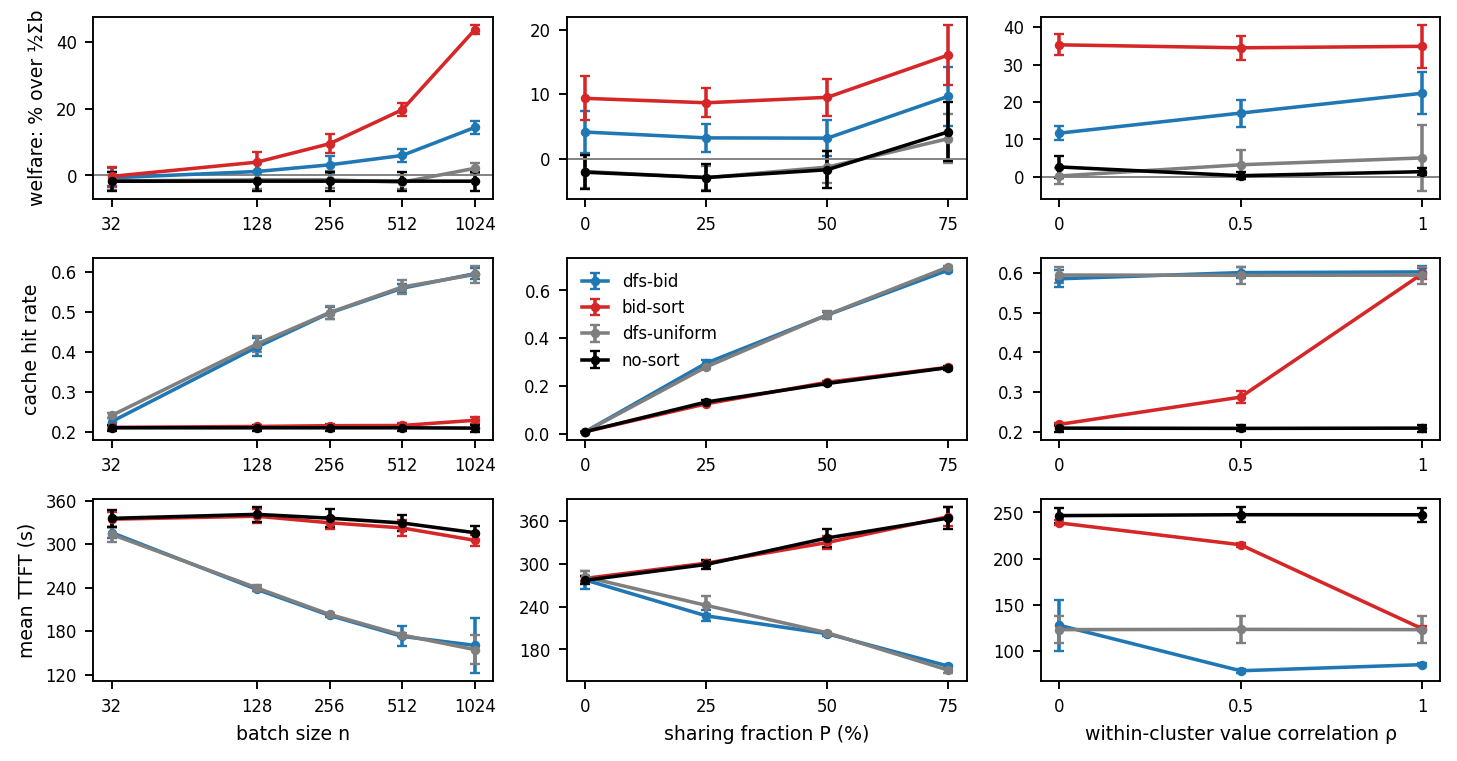}
    \caption{Stars topology. Welfare, cache hit rate, and mean TTFT, varying: batch size $n$ (left), sharing fraction $P$ with $n{=}256$ (middle), and
    within-cluster bid correlation $\rho$ (right).}
    \label{fig:grid}
\end{figure}

\textbf{Experiment Setup.} We evaluate our algorithm as a layer on
top of the SGLang~\citep{zheng2024sglang} inference system across
several semi-synthetic workloads, on a single 80GB H100 serving a 32B
LLM~\citep{guo2025deepseek}. We assume the server is saturated: a
waiting queue holds $Q=1024$ queries, which we split into batches of
$n$ before ordering with our algorithm. When $n \neq Q$ we do not flush the KV cache in between batches. We
configure SGLang to process queries first-come-first-served to preserve their order.\looseness-1

\textbf{Algorithms.} \dfsbid\ is the average-bid DFS of
Algorithm~\ref{alg:auction} with VCG payments as in
Section~\ref{sec:VCG}. \dfsuniform\ runs the
same DFS with all bids equal (meant to replicate SGLang's longest-prefix-match policy); \bidsort\ sorts queries by bid and
ignores the prefix tree (it is the unconstrained welfare maximizer in
Proposition~\ref{prop:dfs-welfare-loss}); \nosort\ serves queries in
arrival order. Values are synthetic, drawn i.i.d.\ (unless otherwise noted) from
$1+\mathrm{Pareto}(1.5)$, and bids are truthful. Every number reported
is a mean of three trials with 1-standard-error bars.

\textbf{Workloads.} We construct semi-synthetic workloads from real
LLM evaluation data, covering three basic radix tree topologies: stars,
chains, and 3-ary trees. For stars, \mixsharename\ mixes
SQuAD~\citep{rajpurkar2016squad} questions sharing an article as a
prefix with GSM8K~\citep{cobbe2021gsm8k} questions sharing few-shot examples; in \mixshare{P}, a $P$\% of the queries share prefixes and the rest are singletons. For chains we replay SWE-agent~\citep{yang2024sweagent} trajectories. 
For trees, we run tree-of-thought~\citep{yao2023tree} in a short variant and a long
variant over articles of 3--5k tokens. Prompts are $\approx$1k tokens
on stars, 3k on chains, and 160 or 4k on the two tree variants, with
$60$--$97\%$ of tokens shared (Appendix~\ref{app:workloads}). The
model's weights leave a 23k-token KV cache;
the short trees
fit in the cache entirely and serve as a control, while the other workloads
force cache eviction.

\begin{figure}[t]
    \centering
    \includegraphics[width=\linewidth]{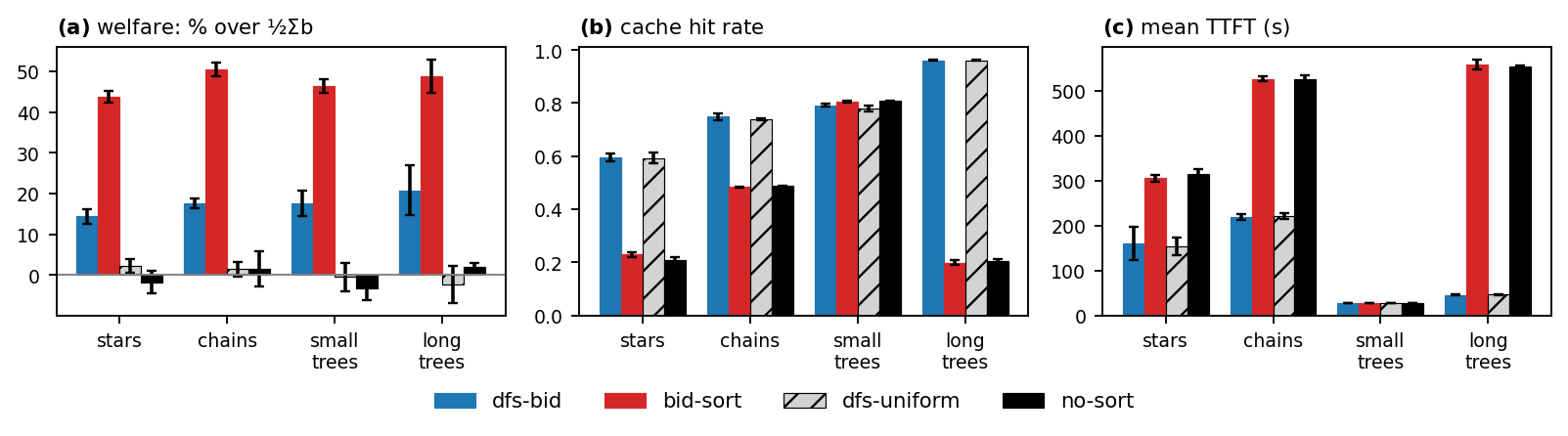}
    \caption{Tree topologies at $n=Q$. When the workload exceeds the
    cache (all topologies except small trees), \bidsort\ pays for its welfare lead (left) in cache misses (middle)
    and latency (right).}
    \label{fig:shapes}
    \vspace{-10pt}
\end{figure}

\begin{figure}[t]
    \centering
    \includegraphics[width=\linewidth]{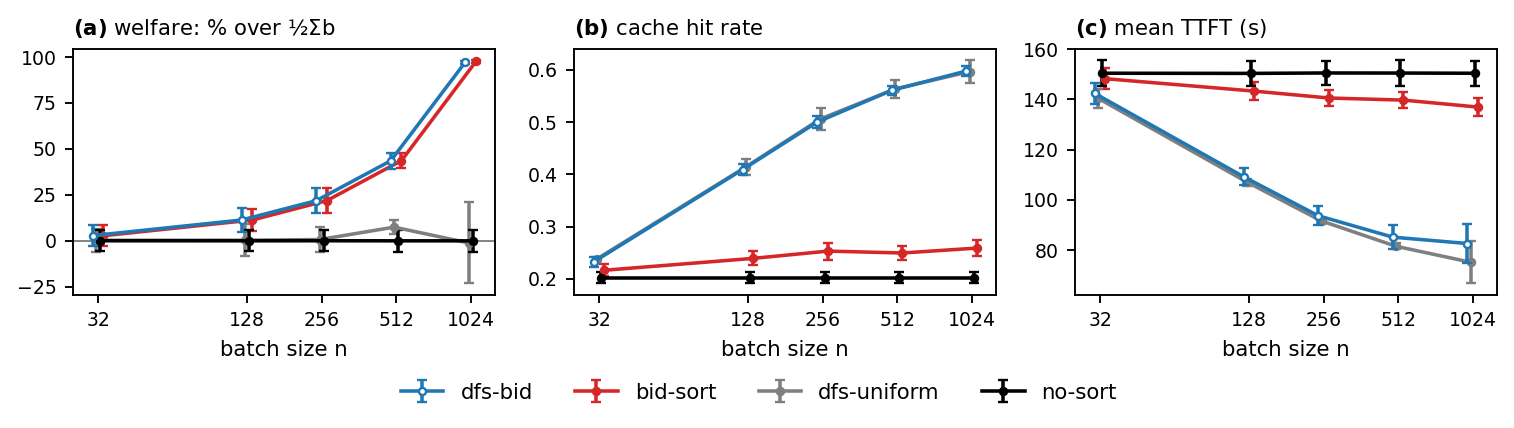}
    \caption{One cluster has non-zero bids, everyone else bids zero (\mixshare{50}).
    \dfsbid\ matches \bidsort's welfare (left; curves offset slightly for readability) while
    keeping \dfsuniform's hit rate (middle) and latency (right).}
    \label{fig:juice}
\end{figure}

\textbf{Results.} 
We measure welfare on the realized order,
i.e., with queries ranked by TTFT, and plot it as the percentage above
$\tfrac12\sum_r b_r$ (this quantity is the expected welfare of a uniformly random DFS order). We
also report the cache hit rate as reported by SGLang and the mean TTFT over the queue.\looseness-1

\begin{wrapfigure}[11]{r}{0.45\textwidth}
    \centering
    \vspace{-12pt}
    \includegraphics[height=3.9cm]{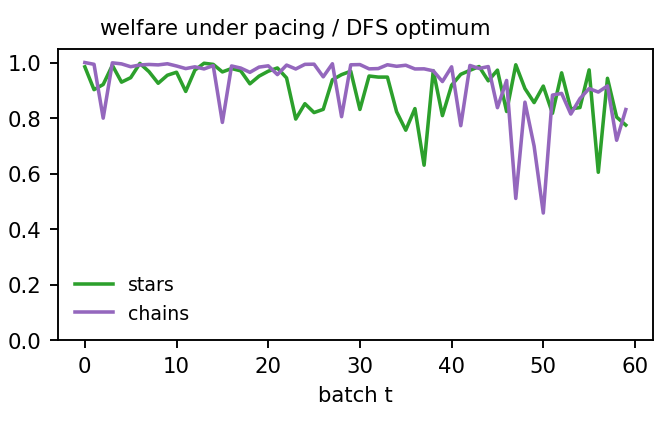}
    \vspace*{-4pt}
    \caption{Welfare under autobidding when $T=60$. Details are in Appendix~\ref{app:autobid}.}
    \label{fig:autobid_main}
\end{wrapfigure}

In every experiment, \dfsbid\ and \dfsuniform\
have the same cache hit rate and the same TTFT, but \dfsbid\ has substantially higher
welfare. \bidsort\ has higher welfare, but
it interleaves queries that share a prefix, significantly reducing cache hit rate. 
At $n=Q$, \dfsbid's  welfare is within $19$--$22\%$ of \bidsort's welfare, well under the $50\%$ that
Proposition~\ref{prop:dfs-welfare-loss} allows.

The first column of Figure~\ref{fig:grid} varies the batch size on
\mixshare{50}. With more queries in a batch there are more shared
prefixes to group, so both DFS orders gain cache hits and get faster
with $n$ while \bidsort\ and \nosort\ do not. In terms of welfare,
\dfsbid's lead over \dfsuniform\ grows from $1\%$ at
$n=32$, with few shared prefixes per batch, to $12\%$ at $n=Q$. The second column varies the sharing fraction at
$n=256$ with the same result: more sharing helps the DFS orders and
hurts the flat ones, and \dfsbid\ stays $4$--$7\%$ above \dfsuniform.
The third column varies $\rho$, how much values are correlated within clusters.
\bidsort~is substantially slower than \dfsbid~until $\rho=1$,
when cluster members are naturally grouped by having the same bid;
even then, \dfsbid\ achieves competitive welfare with better TTFT.

Figure~\ref{fig:shapes} repeats the comparison at $n=Q$ across tree
topologies. \bidsort's TTFT is $2.4\times$ \dfsbid's on chains and
$12\times$ on long trees. The short trees are the exception since they fit
in the cache, so nothing is ever evicted and every order gets the same
hit rate and the same TTFT regardless of how it schedules. This is the one case where the
DFS restriction buys nothing and costs $20\%$ of welfare.\looseness-1

Figure~\ref{fig:juice} shows the case most favorable to our mechanism
(and perhaps the most realistic): one cluster of queries with a shared
prefix has non-zero bids, and every other request bids zero. Serving that cluster
first is both the unconstrained optimum and a DFS order, so \dfsbid\
matches \bidsort's welfare at every batch size while keeping
\dfsuniform's hit rate, and \bidsort\ pays about $1.5\times$ the TTFT.\looseness-1

Finally, in Figure~\ref{fig:autobid_main} we simulate the pacing autobidders of Algorithm~\ref{alg:ogd}
on the same trees. 
Autobidding in repeated auctions causes welfare to decrease by about $10\%$ on average compared to the single-shot setting, and welfare is always above what Theorem~\ref{thm:condition} predicts.\looseness-1
\section{Related Work}\label{sec:related}

\textbf{Prefix-Aware LLM Serving.} 
Modern LLM serving frameworks improve request throughput via continuous batching and efficient management of the KV cache. 
vLLM~\citep{kwon2023efficient} uses ideas from virtual memory and paging to reduce KV cache fragmentation. 
SGLang~\citep{zheng2024sglang} stores requests in a radix tree and uses longest prefix matching to increase cache reuse across requests. 
Preble~\citep{srivatsa2025preble} jointly optimizes prefix reuse and load balancing across multiple GPUs, while BatchLLM~\citep{zheng2026batchllm} groups requests with common prefixes offline and coordinates their execution. 
\citet{dexter2026llm} study optimal request scheduling under a certain arrival pattern when cached requests are stored in a radix tree, and they propose a new family of scheduling policies which interpolate between first-come-first-served and longest prefix matching. 
Llumnix~\cite{sun2024llumnix} and QoServe~\cite{goel2026qoserve} support differentiated latency objectives and request prioritization. 
Another related line of work studies fairness in LLM serving~\citep{sheng2024fairness, cao2025locality}. 
Our inference auction is the first approach that prioritizes economic efficiency.\looseness-1

\textbf{Priority Auctions and Queues.}
The idea of allowing delay-sensitive users to purchase queue priority has a long history~\citep[e.g.,][]{kleinrock1967optimum, glazer1986stable, hassin1995decentralized, mitra2001mechanism, afeche2004pricing, kittsteiner2005priority, heydenreich2008optimal}. 
LLM inference differs from these classical models because inference requests are not computationally independent from one another. 
To the best of our knowledge, our work is the first to design incentive-compatible priority mechanisms for LLM inference serving. 

\textbf{Autobidding Agents.}
Autobidding has been studied extensively in online ad auctions; see~\citet{aggarwal2024auto} for a recent survey. 
\citet{balseiro2019learning} provide an adaptive pacing algorithm which scales bids according to realized expenditure. 
\citet{balseiro2020dual} generalize their algorithm to online mirror descent and weaken several theoretical assumptions. 
Subsequent work studies social welfare and learning dynamics when autobidders interact in repeated ad auctions~\citep[e.g.][]{balseiro2021robust, deng2021towards, gaitonde2022budget, lucier2023autobidders, paes2024complex, anagnostides2026chaos}. 
While our pacing autobidding algorithm is similar to some that are typically used in ad auctions, our welfare guarantees leverage the specific structure of the radix tree and are unique to our LLM inference setting. 

\textbf{LLMs and Incentives.}
Our results are broadly related to the growing body of work at the intersection of game theory, mechanism design, and generative AI. 
%
%
\citet{chen2026error},~\citet{guo2026pricing}, and~\citet{liu2026iemas} study pricing and incentives for LLM routing, while~\citet{harris2026context} and~\citet{anagnostides2026compensation} design incentive-aware credit assignment schemes for AI-generated content. 
\citet{velasco2026test} study social inefficiency resulting from model providers' incentives to increase test-time compute in order to increase revenue; they propose having model providers bid in a reverse second-price auction to resolve this inefficiency. 
There is also a line of work which designs auctions for ad placements in LLM-generated text~\citep{duetting2024mechanism, dubey2024auctions, hajiaghayi2024ad}.
Finally,~\citet{ji2026incentivizing} also combine an inference system with auction design in the context of distributed computing; specifically, procuring edge devices for speculative decoding. 
\looseness-1

\section{Conclusion}

We design the first {\em inference auction}, a mechanism for allocating LLM serving capacity during times of high request volume. 
Our inference auction enjoys the following key properties: (1) forcing the winning schedule to be a DFS schedule ensures cache optimality, (2) welfare maximization corresponds to a simple DFS comparator that orders radix subtrees by average bid, and (3) VCG payments, which incentivize truthful bidding, can be implemented with negligible computational overhead.
%
%
%
We also design a pacing autobidder that learns how to bid in repeated inference auctions under a budget constraint. 
Our autobidder dynamically adjusts bids based on realized spending, and achieves sublinear regret relative to the optimal-in-hindsight bidding policy. 
When users in a batch run pacing autobidders independently, we 
bound the resulting welfare loss 
by the differences in the effective pacing factors of sibling subtrees. 
In experiments, our inference auction improves social welfare while retaining the caching and latency advantages of SGLang.
Our work demonstrates how economic objectives can be incorporated directly into LLM serving infrastructure without sacrificing performance.\looseness-1

There are several exciting directions for future research. 
Serving large-scale systems may require joint orchestration of auctions, request routing, and cache management across multiple inference workers. Support for richer latency preferences is also an interesting direction. 
%
%
Finally, collusion-proof inference auction design is important to prevent strategic coordination across users.\looseness-1 

\section*{Acknowledgments} 
We would like to thank Annie Ulichney for helpful comments and suggestions. 
This work was funded by the European Union (ERC-2022-SYG-OCEAN-101071601), the NSF Institute for Foundations of Machine Learning under grant CCF-2505865, the National Science Foundation under grant CCF-2145898, the Office of Naval Research under grant N00014-24-1-2159, a Google Research Scholar Award, an Alfred P. Sloan fellowship, and a Schmidt Science AI2050 fellowship. 
Views and opinions expressed are however those of the author(s) only and do not necessarily reflect those of the European Union or the European Research Council
Executive Agency. Neither the European Union nor the granting authority can be held responsible for them.\looseness-1

\bibliographystyle{plainnat}
\bibliography{refs}

@misc{openai_priority_processing,
  author  = {{OpenAI}},
  title   = {Priority Processing},
  year    = {2026},
  url     = {https://developers.openai.com/api/docs/guides/priority-processing},
  urldate = {2026-07-02},
}

@misc{anthropic_service_tiers,
  author  = {{Anthropic}},
  title   = {Service Tiers},
  year    = {2026},
  url     = {https://platform.claude.com/docs/en/api/service-tiers},
  urldate = {2026-07-02},
}

@misc{aws_bedrock_service_tiers,
  author  = {{Amazon Web Services}},
  title   = {Optimize Costs with {A}mazon {B}edrock Provisioned Throughput and Service Tiers},
  year    = {2026},
  url     = {https://docs.aws.amazon.com/bedrock/latest/userguide/service-tiers-inference.html},
  urldate = {2026-07-02},
}

@misc{google2026priority,
  author = {{Google}},
  title  = {{Gemini API}: Priority Inference},
  year   = {2026},
  url    = {https://ai.google.dev/gemini-api/docs/priority-inference},
}

@misc{xai2026priority,
  author = {{xAI}},
  title  = {Priority Processing},
  year   = {2026},
  url    = {https://docs.x.ai/developers/advanced-api-usage/priority-processing},
}

@inproceedings{kwon2023efficient,
  title={Efficient memory management for large language model serving with {P}aged{A}ttention},
  author={Kwon, Woosuk and Li, Zhuohan and Zhuang, Siyuan and Sheng, Ying and Zheng, Lianmin and Yu, Cody Hao and Gonzalez, Joseph and Zhang, Hao and Stoica, Ion},
  booktitle={Proceedings of the 29th Symposium on Operating Systems Principles},
  pages={611--626},
  year={2023}
}

@article{zheng2024sglang,
  title={{SGL}ang: Efficient execution of structured language model programs},
  author={Zheng, Lianmin and Yin, Liangsheng and Xie, Zhiqiang and Sun, Chuyue and Huang, Jeff and Yu, Cody H and Cao, Shiyi and Kozyrakis, Christos and Stoica, Ion and Gonzalez, Joseph E and others},
  journal={Advances in Neural Information Processing Systems},
  volume={37},
  pages={62557--62583},
  year={2024}
}

@article{dexter2026llm,
  title={{LLM} query scheduling with prefix reuse and latency constraints},
  author={Dexter, Gregory and Tang, Shao and Fatahi, Ata and Song, Qingquan and Dharamsi, Tejas and Gupta, Aman},
  journal={Advances in Neural Information Processing Systems},
  volume={38},
  pages={94356--94379},
  year={2026}
}

@article{aggarwal2024auto,
  title={Auto-bidding and auctions in online advertising: A survey},
  author={Aggarwal, Gagan and Badanidiyuru, Ashwinkumar and Balseiro, Santiago R and Bhawalkar, Kshipra and Deng, Yuan and Feng, Zhe and Goel, Gagan and Liaw, Christopher and Lu, Haihao and Mahdian, Mohammad and others},
  journal={ACM SIGecom Exchanges},
  volume={22},
  number={1},
  pages={159--183},
  year={2024},
  publisher={ACM New York, NY, USA}
}

@article{krishna1998efficient,
  title={Efficient mechanism design},
  author={Krishna, Vijay and Perry, Motty},
  journal={Available at SSRN 64934},
  year={1998}
}

@inproceedings{prasad2026weakest,
  title={Weakest bidder types and new core-selecting combinatorial auctions},
  author={Prasad, Siddharth and Balcan, Maria Florina and Sandholm, Tuomas},
  booktitle={Proceedings of the AAAI Conference on Artificial Intelligence},
  volume={40},
  pages={17206--17214},
  year={2026}
}

@article{balcan2023bicriteria,
  title={Bicriteria multidimensional mechanism design with side information},
  author={Balcan, Maria-Florina and Prasad, Siddharth and Sandholm, Tuomas},
  journal={Advances in Neural Information Processing Systems},
  volume={36},
  pages={40832--40852},
  year={2023}
}

@inproceedings{balseiro2020dual,
  title={Dual mirror descent for online allocation problems},
  author={Balseiro, Santiago and Lu, Haihao and Mirrokni, Vahab},
  booktitle={International Conference on Machine Learning},
  pages={613--628},
  year={2020},
  organization={PMLR}
}

@article{Vickrey61:Counterspeculation,
  title={Counterspeculation, auctions, and competitive sealed tenders},
  author={Vickrey, William},
  journal={Journal of Finance},
  volume={16},
  number={1},
  pages={8--37},
  year={1961},
  publisher={JSTOR}
}

@article{Clarke71:Multipart,
  title={Multipart pricing of public goods},
  author={Clarke, Edward H},
  journal={Public Choice},
  pages={17--33},
  year={1971},
  publisher={JSTOR}
}

@article{Groves73:Incentives,
  title={Incentives in teams},
  author={Groves, Theodore},
  journal={Econometrica},
  pages={617--631},
  year={1973},
  publisher={JSTOR}
}

@inproceedings{sheng2024fairness,
  title={Fairness in serving large language models},
  author={Sheng, Ying and Cao, Shiyi and Li, Dacheng and Zhu, Banghua and Li, Zhuohan and Zhuo, Danyang and Gonzalez, Joseph E and Stoica, Ion},
  booktitle={18th USENIX Symposium on Operating Systems Design and Implementation (OSDI 24)},
  pages={965--988},
  year={2024}
}

@inproceedings{heydenreich2008optimal,
  title={Optimal mechanisms for single machine scheduling},
  author={Heydenreich, Birgit and Mishra, Debasis and M{\"u}ller, Rudolf and Uetz, Marc},
  booktitle={International Workshop on Internet and Network Economics},
  pages={414--425},
  year={2008},
  organization={Springer}
}

@inproceedings{srivatsa2025preble,
  title={Preble: Efficient distributed prompt scheduling for {LLM} serving},
  author={Srivatsa, Vikranth and He, Zijian and Abhyankar, Reyna and Li, Dongming and Zhang, Yiying},
  booktitle={International Conference on Learning Representations},
  volume={2025},
  pages={37057--37082},
  year={2025}
}

@article{zheng2026batchllm,
  title={Batch{LLM}: Optimizing large batched {LLM} inference with global prefix sharing and throughput-oriented token batching},
  author={Zheng, Zhen and Ji, Xin and Fang, Taosong and Zhou, Fanghao and Liu, Chuanjie and Peng, Gang},
  journal={Proceedings of Machine Learning and Systems},
  volume={8},
  pages={1742--1756},
  year={2026}
}

@article{cao2025locality,
  title={Locality-aware fair scheduling in {LLM} serving},
  author={Cao, Shiyi and Wang, Yichuan and Mao, Ziming and Hsu, Pin-Lun and Yin, Liangsheng and Xia, Tian and Li, Dacheng and Liu, Shu and Zhang, Yineng and Zhou, Yang and others},
  journal={arXiv preprint arXiv:2501.14312},
  year={2025}
}

@article{kleinrock1967optimum,
  title={Optimum bribing for queue position},
  author={Kleinrock, Leonard},
  journal={Operations Research},
  volume={15},
  number={2},
  pages={304--318},
  year={1967},
  publisher={INFORMS}
}

@article{hassin1995decentralized,
  title={Decentralized regulation of a queue},
  author={Hassin, Rafael},
  journal={Management Science},
  volume={41},
  number={1},
  pages={163--173},
  year={1995},
  publisher={INFORMS}
}

@article{afeche2004pricing,
  title={Pricing and priority auctions in queueing systems with a generalized delay cost structure},
  author={Afeche, Philipp and Mendelson, Haim},
  journal={Management Science},
  volume={50},
  number={7},
  pages={869--882},
  year={2004},
  publisher={INFORMS}
}

@article{kittsteiner2005priority,
  title={Priority auctions and queue disciplines that depend on processing time},
  author={Kittsteiner, Thomas and Moldovanu, Benny},
  journal={Management Science},
  volume={51},
  number={2},
  pages={236--248},
  year={2005},
  publisher={INFORMS}
}

@article{mitra2001mechanism,
  title={Mechanism design in queueing problems},
  author={Mitra, Manipushpak},
  journal={Economic Theory},
  volume={17},
  number={2},
  pages={277--305},
  year={2001},
  publisher={Springer}
}

@article{balseiro2019learning,
  title={Learning in repeated auctions with budgets: Regret minimization and equilibrium},
  author={Balseiro, Santiago R and Gur, Yonatan},
  journal={Management Science},
  volume={65},
  number={9},
  pages={3952--3968},
  year={2019},
  publisher={INFORMS}
}

@article{lucier2023autobidders,
  title={Autobidders with budget and {ROI} constraints: Efficiency, regret, and pacing dynamics},
  author={Lucier, Brendan and Pattathil, Sarath and Slivkins, Aleksandrs and Zhang, Mengxiao},
  journal={arXiv preprint arXiv:2301.13306},
  year={2023}
}

@article{balseiro2021robust,
  title={Robust auction design in the auto-bidding world},
  author={Balseiro, Santiago and Deng, Yuan and Mao, Jieming and Mirrokni, Vahab and Zuo, Song},
  journal={Advances in Neural Information Processing Systems},
  volume={34},
  pages={17777--17788},
  year={2021}
}

@inproceedings{deng2021towards,
  title={Towards efficient auctions in an auto-bidding world},
  author={Deng, Yuan and Mao, Jieming and Mirrokni, Vahab and Zuo, Song},
  booktitle={Proceedings of the Web Conference 2021},
  pages={3965--3973},
  year={2021}
}

@article{gaitonde2022budget,
  title={Budget pacing in repeated auctions: Regret and efficiency without convergence},
  author={Gaitonde, Jason and Li, Yingkai and Light, Bar and Lucier, Brendan and Slivkins, Aleksandrs},
  journal={arXiv preprint arXiv:2205.08674},
  year={2022}
}

@article{harris2026context,
  title={In-context credit assignment via the core},
  author={Harris, Keegan and Prasad, Siddharth and Trockman, Asher},
  journal={arXiv preprint arXiv:2605.06920},
  year={2026}
}

@article{chen2026error,
  title={Error-Aware Reverse Auction Mechanism for Large Language Model Routing},
  author={Chen, Haolong and Xin, Zhengyuan and Zhang, Liang and Xue, Lei and Zhu, Guangxu},
  journal={arXiv preprint arXiv:2608.12719},
  year={2026}
}

@article{liu2026iemas,
  title={IEMAS: An Incentive-Efficiency Routing Framework for Open Agentic Web Ecosystems},
  author={Liu, Hongze and Guo, Chang and Li, Yingzeng and Wang, Mengru and Lou, Jiong and Yuan, Shijing and Zhou, Hefeng and Wu, Chentao and Li, Jie},
  journal={arXiv preprint arXiv:2603.17302},
  year={2026}
}

@inproceedings{duetting2024mechanism,
  title={Mechanism design for large language models},
  author={Duetting, Paul and Mirrokni, Vahab and Paes Leme, Renato and Xu, Haifeng and Zuo, Song},
  booktitle={Proceedings of the ACM Web Conference 2024},
  pages={144--155},
  year={2024}
}

@inproceedings{dubey2024auctions,
  title={Auctions with {LLM} summaries},
  author={Dubey, Avinava and Feng, Zhe and Kidambi, Rahul and Mehta, Aranyak and Wang, Di},
  booktitle={Proceedings of the 30th ACM SIGKDD Conference on Knowledge Discovery and Data Mining},
  pages={713--722},
  year={2024}
}

@article{hajiaghayi2024ad,
  title={Ad auctions for {LLM}s via retrieval augmented generation},
  author={Hajiaghayi, MohammadTaghi and Lahaie, S{\'e}bastien and Rezaei, Keivan and Shin, Suho},
  journal={Advances in Neural Information Processing Systems},
  volume={37},
  pages={18445--18480},
  year={2024}
}

@inproceedings{paes2024complex,
  title={Complex dynamics in autobidding systems},
  author={Paes Leme, Renato and Piliouras, Georgios and Schneider, Jon and Spendlove, Kelly and Zuo, Song},
  booktitle={Proceedings of the 25th ACM Conference on Economics and Computation},
  pages={75--100},
  year={2024}
}

@article{anagnostides2026chaos,
  title={Chaos in Autobidding Auctions},
  author={Anagnostides, Ioannis and Gemp, Ian and Piliouras, Georgios and Spendlove, Kelly},
  journal={arXiv preprint arXiv:2602.09118},
  year={2026}
}

@article{anagnostides2026compensation,
  title={Compensation Design},
  author={Anagnostides, Ioannis and Bhawalkar, Kshipra and Liaw, Christopher and Mehta, Aranyak and Leme, Renato Paes and Teng, Yifeng and Velegkas, Grigoris and Zheng, Weiqiang},
  journal={arXiv preprint arXiv:2607.14438},
  year={2026}
}

@article{glazer1986stable,
  title={Stable priority purchasing in queues},
  author={Glazer, Amihai and Hassin, Refael},
  journal={Operations Research Letters},
  volume={4},
  number={6},
  pages={285--288},
  year={1986},
  publisher={Elsevier}
}

@inproceedings{rajpurkar2016squad,
  title={{SQuAD}: 100,000+ Questions for Machine Comprehension of Text},
  author={Rajpurkar, Pranav and Zhang, Jian and Lopyrev, Konstantin and Liang, Percy},
  booktitle={Proceedings of the 2016 Conference on Empirical Methods in Natural Language Processing},
  year={2016}
}

@article{cobbe2021gsm8k,
  title={Training Verifiers to Solve Math Word Problems},
  author={Cobbe, Karl and Kosaraju, Vineet and Bavarian, Mohammad and Chen, Mark and Jun, Heewoo and Kaiser, Lukasz and Plappert, Matthias and Tworek, Jerry and Hilton, Jacob and Nakano, Reiichiro and Hesse, Christopher and Schulman, John},
  journal={arXiv preprint arXiv:2110.14168},
  year={2021}
}

@inproceedings{yang2024sweagent,
  title={{SWE}-agent: Agent-Computer Interfaces Enable Automated Software Engineering},
  author={Yang, John and Jimenez, Carlos E. and Wettig, Alexander and Lieret, Kilian and Yao, Shunyu and Narasimhan, Karthik and Press, Ofir},
  booktitle={Advances in Neural Information Processing Systems},
  year={2024}
}

@inproceedings{yao2023tree,
  title={Tree of Thoughts: Deliberate Problem Solving with Large Language Models},
  author={Yao, Shunyu and Yu, Dian and Zhao, Jeffrey and Shafran, Izhak and Griffiths, Thomas L. and Cao, Yuan and Narasimhan, Karthik},
  booktitle={Advances in Neural Information Processing Systems},
  year={2023}
}

@article{guo2025deepseek,
  title={{DeepSeek-R1}: Incentivizing Reasoning Capability in {LLMs} via Reinforcement Learning},
  author={{DeepSeek-AI}},
  journal={arXiv preprint arXiv:2501.12948},
  year={2025}
}

@article{qwen2025qwen25,
  title={Qwen2.5 Technical Report},
  author={{Qwen Team}},
  journal={arXiv preprint arXiv:2412.15115},
  year={2025}
}

@article{velasco2026test,
  title={Test-Time Compute Games},
  author={Velasco, Ander Artola and Rontogiannis, Dimitrios and Tsirtsis, Stratis and Gomez-Rodriguez, Manuel},
  journal={arXiv preprint arXiv:2601.21839},
  year={2026}
}

@inproceedings{guo2026pricing,
  title={Pricing online LLM services with data-calibrated stackelberg routing game},
  author={Guo, Zhendong and Bai, Wenchao and Jin, Jiahui},
  booktitle={Proceedings of the AAAI Conference on Artificial Intelligence},
  volume={40},
  pages={17005--17013},
  year={2026}
}

@inproceedings{ji2026incentivizing,
  title={Incentivizing and Orchestrating Cloud-Edge {LLM} Speculative Decoding via Auctions},
  author={Ji, Mingtao and Jiao, Lei and Tang, Bin and Qu, Zhihao and Chen, Chen and Ye, Baoliu},
  booktitle={2026 IEEE 46th International Conference on Distributed Computing Systems (ICDCS)},
  pages={425--435},
  year={2026},
  organization={IEEE}
}

@inproceedings{sun2024llumnix,
  title={Llumnix: Dynamic scheduling for large language model serving},
  author={Sun, Biao and Huang, Ziming and Zhao, Hanyu and Xiao, Wencong and Zhang, Xinyi and Li, Yong and Lin, Wei},
  booktitle={18th USENIX Symposium on Operating Systems Design and Implementation (OSDI 24)},
  pages={173--191},
  year={2024}
}

@inproceedings{goel2026qoserve,
  title={QoServe: Breaking the silos of {LLM} inference serving},
  author={Goel, Kanishk and Mohan, Jayashree and Kwatra, Nipun and Anupindi, Ravi Shreyas and Ramjee, Ramachandran},
  booktitle={Proceedings of the 31st ACM International Conference on Architectural Support for Programming Languages and Operating Systems, Volume 2},
  pages={1492--1507},
  year={2026}
}

@inproceedings{qin2025mooncake,
  title={Mooncake: Trading more storage for less computation—a $\{$KVCache-centric$\}$ architecture for serving $\{$LLM$\}$ chatbot},
  author={Qin, Ruoyu and Li, Zheming and He, Weiran and Cui, Jialei and Ren, Feng and Zhang, Mingxing and Wu, Yongwei and Zheng, Weimin and Xu, Xinran},
  booktitle={23rd USENIX Conference on File and Storage Technologies (FAST 25)},
  pages={155--170},
  year={2025}
}

\newpage
\appendix
\section{Appendix for Section~\ref{sec:env}: Cache-Optimal Schedules and Our Bidding Environment}\label{app:env}

\begin{figure}[h]
\centering
\begin{tikzpicture}[
  tok/.style={draw, fill=white, inner sep=3pt, font=\small},
  br/.style={circle, fill, inner sep=1.5pt},
  leaf/.style={circle, draw, inner sep=1pt, font=\scriptsize},
  edge/.style={line width=1.1pt},
]
  \node[tok] (s)  at (0,0)    {Summarize};
  \node[br]  (d0) at (0,-0.8) {};
  \draw[edge] (0,0.8) -- (s) -- (d0);

  \node[tok] (art) at (-2.8,-1.7) {the article};
  \node[tok] (em)  at ( 2.8,-1.7) {this email};
  \node[br]  (d1)  at (-2.8,-2.5) {};
  \node[br]  (d2)  at ( 2.8,-2.5) {};
  \draw[edge] (d0) -- (art) -- (d1);
  \draw[edge] (d0) -- (em)  -- (d2);

  \node[tok] (cats)  at (-5.0,-3.4) {about cats};
  \node[tok] (bul)   at (-2.8,-3.4) {as bullets};
  \node[tok] (fr)    at (-0.6,-3.4) {in French};
  \node[tok] (brf)   at ( 1.7,-3.4) {briefly};
  \node[tok] (formy) at ( 3.9,-3.4) {for my};
  \draw[edge] (d1) -- (cats);
  \draw[edge] (d1) -- (bul);
  \draw[edge] (d1) -- (fr);
  \draw[edge] (d2) -- (brf);
  \draw[edge] (d2) -- (formy);

  \node[leaf] (l1) at (-5.0,-4.3) {1};
  \node[leaf] (l2) at (-2.8,-4.3) {2};
  \node[leaf] (l3) at (-0.6,-4.3) {3};
  \node[leaf] (l4) at ( 1.7,-4.3) {4};
  \draw[edge] (cats) -- (l1);
  \draw[edge] (bul)  -- (l2);
  \draw[edge] (fr)   -- (l3);
  \draw[edge] (brf)  -- (l4);

  \node[br]  (d3)   at (3.9,-4.2) {};
  \node[tok] (boss) at (3.1,-5.1) {boss};
  \node[tok] (team) at (4.7,-5.1) {team};
  \node[leaf] (l5)  at (3.1,-6.0) {5};
  \node[leaf] (l6)  at (4.7,-6.0) {6};
  \draw[edge] (formy) -- (d3);
  \draw[edge] (d3) -- (boss) -- (l5);
  \draw[edge] (d3) -- (team) -- (l6);

  \node[anchor=north west, font=\small] at (-6.0,-4.9) {%
    \begin{tabular}{@{}r@{\hspace{0.5em}}l@{}}
      1 & Summarize the article about cats \\
      2 & Summarize the article as bullets \\
      3 & Summarize the article in French \\
      4 & Summarize this email briefly \\
      5 & Summarize this email for my boss \\
      6 & Summarize this email for my team \\
    \end{tabular}};
\end{tikzpicture}
\caption{A request radix tree. Each request is a root-to-leaf path; requests
with a common prefix share the path from the root to their lowest common
ancestor, and non-branching paths are contracted into a single edge. The maximum branching factor of this tree is $3$. Requests in the same subtree are easier to store together in the KV cache, as their shared prefixes only need to be stored once.}
\label{fig:radix-tree}
\end{figure}
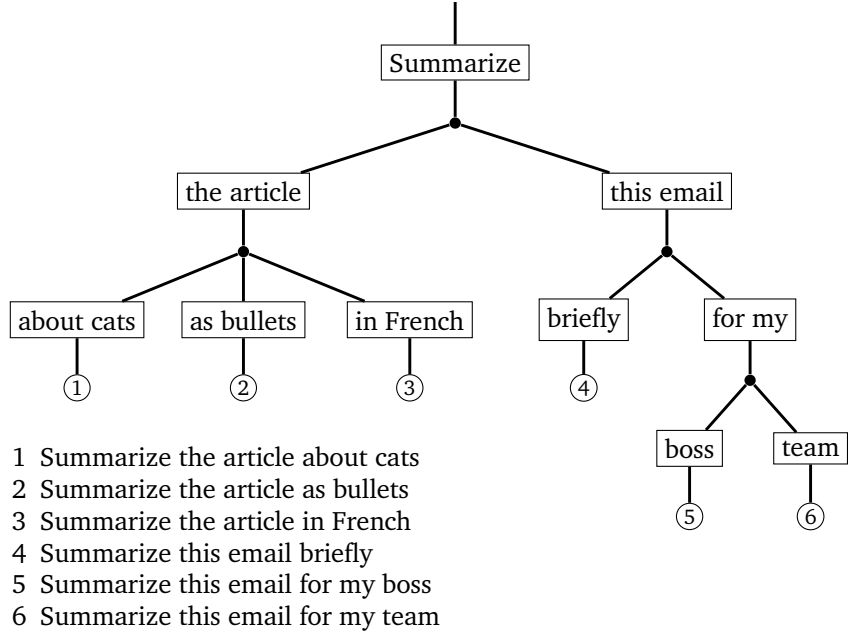

\propwelfare*

\begin{proof}
    Let $\sigma\in\cD(\sT)$ be any feasible schedule and let $\pi$ be the reverse of $\sigma$. 
    We have $\pi\in\cD(\sT)$ since $\pi$ is attainable by reversing all child orderings in the DFS traversal of $\sT$ corresponding to $\sigma$. 
    For any user $i$ and request $r$, we have $\vx_i(\sigma)[r] + \vx_i(\pi)[r] = 1$, so $W(\sigma;\boldv_{1:m}) + W(\pi;\boldv_{1:m}) = \sum_{i=1}^m \|\boldv_i\|_1$. 
    Therefore, at least one of $\sigma$ or $\pi$ generates welfare at least $\frac{1}{2}\sum_{i=1}^m \|\boldv_i\|_1$, and so $\max_{\sigma\in\cD(\sT)}W(\sigma;\boldv_{1:m})\ge\frac{1}{2}\sum_{i=1}^m \|\boldv_i\|_1 \geq \frac{1}{2}\max_{\sigma\in\cS_n}W(\sigma;\boldv_{1:m})$.
\end{proof}
\section{Appendix for Section~\ref{sec:dfs}: Welfare-Optimal Inference Auctions}\label{app:dfs}

\paragraph{Discussion on VCG Payments.} 
The~\citet{Clarke71:Multipart} pivot term $h_i(\vb_{-i})$ in VCG payments is often (informally) framed as the maximum welfare the bidders excluding $i$ could obtain when bidder $i$ is removed from the auction, while we define it by zeroing out that bidder's bids. In typical multi-item multi-bidder auction settings, there is no distinction. In our setting, actually evicting $i$'s requests would alter the radix tree of requests and change the ground set of feasible schedules to allow requests to go un-served. In addition to the fact that all requests should eventually be served, a payment scheme based on eviction would incur outward payments made by the system to the users. Indeed, any zero-bidder would have to be paid because the maximum welfare enjoyed by the remaining bidders decreases when that zero-bidder is evicted, since each remaining bidder would win lower fractional priority.\looseness-1

In our setting, the correct formulation of VCG involves the minimum value of the maximum welfare the bidders excluding $i$ could obtain over all possible ``types'' (bids) for bidder $i$, which is attained when $i$ is a zero bidder.
This is precisely the reasoning folded into the enhancement of VCG based on ``weakest types''~\citep[see, e.g.,][]{krishna1998efficient,balcan2023bicriteria,prasad2026weakest}. 

For completeness, we include below the standard arguments demonstrating that VCG in our setting is indeed efficient, incentive compatible, individually rational, and incurs no positive transfers.

Under VCG prices, efficiency follows from Theorem~\ref{thm:welfare}. 
To see why these payments incentivize truthful bidding, fix the bids $\vb_{-i}$ of the other users.
If user $i$ bids according to $\vb_i$, his utility can be written as 
\begin{equation*}
V_i(\Bar{A}(\vb_i,\vb_{-i})) - p_i^{\vcg}(\vb_i,\vb_{-i})
= W(\Bar{A}(\vb_i,\vb_{-i});(\boldv_i,\vb_{-i})) - h_i(\vb_{-i}).
\end{equation*}
$h_i(\vb_{-i})$ does not depend on $\vb_i$, and $W(\Bar{A}(\vb_i,\vb_{-i});(\boldv_i,\vb_{-i}))$ is maximized when $\vb_i = \boldv_i$ by Theorem~\ref{thm:welfare}. 
Therefore, it is a (weakly) dominant strategy for each user to bid truthfully, independent of the bids of the other users. 
Individual rationality holds because user $i$'s values and priorities are non-negative. 
Under truthful bidding, user $i$'s utility is $\max_{\sigma \in \cD(\sT)} W(\sigma; (\boldv_i, \vb_{-i})) - h_i(\vb_{-i})$. 
Now consider the schedule that maximizes the other users' welfare when $i$'s bids are zero. 
That schedule remains feasible when $i$ bids truthfully. 
Under it, total welfare is $h_i(\vb_{-i}) + V_i(\sigma) \geq h_i(\vb_{-i})$, since $V_i(\sigma) \geq 0$. 
Finally, no positive transfers follows because $h_i(\vb_{-i})$ is the maximum welfare available to the other users, and is always at least the welfare they receive under the selected schedule. 

\subsection{Algorithm for Computing the Welfare-Maximizing Schedule}

\begin{figure}[t]
\centering
\definecolor{dfsblue}{RGB}{30,64,230}
\definecolor{dfsgray}{RGB}{120,125,135}
\begin{tikzpicture}[
  tok/.style={draw, fill=white, inner sep=3pt, font=\small},
  br/.style={circle, fill, inner sep=1.5pt},
  leaf/.style={circle, draw, inner sep=1pt, font=\scriptsize},
  bid/.style={anchor=west, xshift=1pt, font=\small},
  avg/.style={fill=#1, text=white, rounded corners=2pt, inner sep=2pt,
              font=\scriptsize\bfseries, anchor=west, xshift=4pt},
  edge/.style={line width=1.1pt, draw=dfsgray},
  first/.style={line width=1.4pt, draw=dfsblue, -stealth},
]
  \node[tok] (s)  at (0,0)    {Summarize};
  \node[br]  (d0) at (0,-0.8) {};
  \draw[edge] (0,0.8) -- (s) -- (d0);

  \node[tok] (art) at (-3.3,-1.7) {the article};
  \node[tok] (em)  at ( 2.5,-1.7) {this email};
  \node[br, dfsblue] (d1) at (-3.3,-2.7) {};
  \node[br]          (d2) at ( 2.5,-2.7) {};
  \draw[first] (d0) -- (art);  \draw[first] (art) -- (d1);
  \draw[edge]  (d0) -- (em) -- (d2);
  \node[avg=dfsblue] at (d1) {avg \$6};
  \node[avg=dfsgray] at (d2) {avg \$5};

  \node[tok] (cats) at (-5.2,-3.6) {about cats};
  \node[tok] (fr)   at (-3.3,-3.6) {in French};
  \node[tok] (bul)  at (-1.4,-3.6) {as bullets};
  \draw[first] (d1) -- (cats);
  \draw[edge]  (d1) -- (fr);
  \draw[edge]  (d1) -- (bul);

  \node[tok] (formy) at (1.3,-3.6) {for my};
  \node[tok] (brf)   at (3.8,-3.6) {briefly};
  \draw[first] (d2) -- (formy);
  \draw[edge]  (d2) -- (brf);

  \node[br, dfsblue] (d3) at (1.3,-4.4) {};
  \draw[edge] (formy) -- (d3);
  \node[avg=dfsblue] at (d3) {avg \$6.5};
  \node[tok] (boss) at (0.5,-5.2) {boss};
  \node[tok] (team) at (2.1,-5.2) {team};
  \draw[first] (d3) -- (boss);
  \draw[edge]  (d3) -- (team);

  \foreach \r/\x/\b/\src in {1/-5.2/9/cats, 3/-3.3/6/fr, 2/-1.4/3/bul,
                             5/0.5/8/boss, 6/2.1/5/team, 4/3.8/2/brf} {
    \node[leaf] (l\r) at (\x,-6.1) {\r};
    \node[bid] at (l\r.east) {\$\b};
    \draw[edge] (\src) -- (l\r);
  }

  \draw[-stealth, line width=0.8pt, dfsgray] (-5.8,-6.8) -- (4.6,-6.8)
    node[midway, below, font=\small, text=black] {prefill order};
\end{tikzpicture}
\caption{The welfare-maximizing DFS schedule on an example radix tree. At every internal node, children
are visited in decreasing order of average subtree bid (blue marks the child visited first). The resulting schedule is $1, 3, 2, 5, 6, 4$. Note that
request~5 bids \$8 but is served after request~2, which bids \$3. This is because request~5 belongs to a subtree with a lower average bid, and a DFS schedule cannot
interleave the two subtrees.}
\label{fig:avg-bid-dfs}
\end{figure}
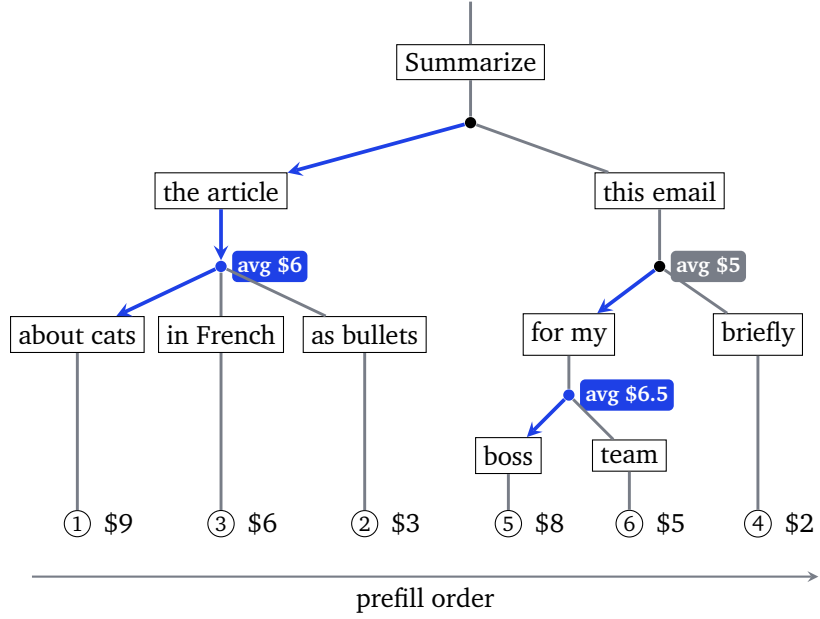

\thmwelfare*

Let $\Phi(\sigma; \vb_{1:m}) \coloneq (n-1)W(\sigma; \vb_{1:m})= \sum_{r=1}^n b_r\cdot |\{j : \sigma(r) < \sigma(j)\}| = \sum_{r = 1}^n\sum_{j\succ_{\sigma} r} b_r$.
Maximizing $W(\cdot;\vb_{1:m})$ over a set of feasible schedules is equivalent to maximizing its un-normalized version $\Phi(\cdot;\vb_{1:m})$. 

\begin{proof}
    Let $\sigma\in\mathcal{D}(\sT)$ be a valid DFS leaf ordering. For $\mathsf{u}$ an internal node (that is, not a leaf node) of $\sT$, let $\mathsf{c}_{1},\ldots,\mathsf{c}_{d(\su)}$ be the children of $\su$, labeled in the DFS order induced by $\sigma$. So, all leaves under child $\mathsf{c}_j$ are visited before all leaves under child $\mathsf{c}_k$ for all $1\le j < k\le d(\su)$.
    The un-normalized welfare contribution from ordered leaf pairs in $\cL(\mathsf{c}_j)\times\cL(\mathsf{c}_k)$ for $j < k$ is therefore $\sum_{i_1\in\cL(\mathsf{c}_j)}\sum_{i_2\in\cL(\mathsf{c}_k)} b_{i_1} = N(\mathsf{c}_k)F_b(\mathsf{c}_j)$. It follows that un-normalized welfare can be written as
    $$\Phi(\sigma;\vb_{1:m}) = \sum_{\su}\sum_{1\le j < k \le d(\su)} N(\mathsf{c}_k)F_b(\mathsf{c}_j).$$
    
    We show that $\Phi(\sigma;\vb_{1:m})$ is maximized by visiting the children of any internal node $\su$ in descending order of $F_b/N$.
    Suppose $\mathsf{c}_{j}, \mathsf{c}_{j+1}\in\cC(\su)$ are adjacent in the DFS traversal with $F_b(\mathsf{c}_j) / N(\mathsf{c}_j) < F_b(\mathsf{c}_{j+1}) / N(\mathsf{c}_{j+1})$. This pair contributes $N(\mathsf{c}_{j+1})F_b(\mathsf{c}_j)$ to $\Phi$. Swapping their order contributes $N(\mathsf{c}_j)F_b(\mathsf{c}_{j+1}) > N(\mathsf{c}_{j+1})F_b(\mathsf{c}_j)$ to welfare without impacting the contribution of any other pair in $\cC(\su)$ (since the relative ordering of $\mathsf{c}_j$ and $\mathsf{c}_{j+1}$ with respect to all other children is preserved).
    Therefore, each inner sum of $\sum_{1\le j < k\le d(\su)}N(\mathsf{c}_k)F_b(\mathsf{c}_j)$ is independently maximized for each $\su$ by visiting $\mathcal{C}(\su)$ in descending order of average bid, so average-bid DFS is a welfare optimal DFS traversal. 

    \textit{Runtime analysis:} A radix tree can be constructed in $O(M+n)=O(M)$ expected time using hash tables for child lookup. 
    Let $\sT = (V,E)$ be the radix tree and let $d(\su)$ denote the number of children of internal node $\su$. 
    Note that $d(\su) \leq \nu$ by definition, and let $I$ be the total number of internal nodes. 
    It takes $O(|V|)$ time to compute $F_b(\su)$ for every node, since each node is processed a constant number of times across all computations. 
    Sorting at each internal node $\su$ takes $O(d(\su) \log(d(\su)))$ time. 
    The time it takes to sort all nodes is $O(\sum_{\su} d(\su) \log(d(\su))) = O(\sum_{\su} d(\su) \log \nu) = O(|E| \log \nu) = O(|V| \log \nu)$. 
    A DFS schedule following this order outputs the order in $O(|V|)$ time. 
    Finally, since a radix tree has no internal node with exactly one child, we have that $n + I - 1 = |E| = \sum_u d(u) \geq 2I$, and so $V \leq 2n - 1$. 
\end{proof}

\subsection{Algorithm for Computing VCG Prices}\label{app:vcg_alg}

In this subsection we abbreviate $F_b$ as $F$ for brevity. Let $[x]_+ \coloneq \max\{x,0\}$.

Fix an internal node $\su$ and let $\child_1,\ldots,\child_d$ denote its children sorted in descending average-bid order, so $$\frac{F(\child_1)}{N(\child_1)}\ge\cdots\ge\frac{F(\child_d)}{N(\child_d)}.$$
Consider some user $i$. Let $$Z_{\child_j}[i]\coloneq\sum_{r\in\cR_i\cap\cL(\child_j)}b_r$$ denote user $i$'s aggregate bid across his requests under child $\child_j$. After zeroing out user $i$'s bids, the total bid under child $\child_j$ is $$G_i(\child_j)\coloneq F(\child_j) - Z_{\child_j}[i].$$

\begin{lemma}
    For any user $i$ and internal node $\su$, the welfare externality at $\su$ due to $i$ is $$E_{i}(\su)\coloneq\sum_{1\le j < k\le d}\left[G_i(\child_k)N(\child_j) - G_i(\child_j)N(\child_k)\right]_{+}.$$ Consequently, $(n-1)p_i^{\vcg} = \sum_{\su} E_i(\su).$
\end{lemma}

\begin{proof}
    In the original descending average-bid order, pair $j < k$ of children contributes $G_i(\child_j)N(\child_k)$ to the other users' welfare. The counterfactual optimal ordering sorts children in decreasing order of $G_i / N$, so $j < k$ contributes $\max\{G_i(\child_j)N(\child_k), G_i(\child_k)N(\child_j)\}$ to welfare. So, the externality---that is, the welfare gain in the counterfactual optimum---is precisely $\left[G_i(\child_k)N(\child_j) - G_i(\child_j)N(\child_k)\right]_{+}$. Summing over all child pairs yields the claimed externality $E_i(\su)$. To compute the VCG price it suffices to sum $E_i(\su)$ over all internal nodes and normalize.
\end{proof}

\paragraph{Preprocessing.}
For $t = 1,\ldots, d$, let $$P_F[t]\coloneq\sum_{k=1}^t F(\child_k);\quad P_N[t]\coloneq\sum_{k=1}^t N(\child_k).$$ Computing these prefix sums is a $O(d(\su))$ run-time operation at each node. So, the overall run-time of prefix sum computation is $O(\sum_{\su} d(\su)) = O(n)$. This computation is done once and the same prefix sums are used in the price calculation for all bidders.

\subsubsection{Computing local externalities}

Let $\cA_{\su}[i]\coloneq\{j\in\{1,\ldots,d\} : Z_{\child_j}[i] > 0\}$ denote the set of affected child indices; enumerate $\cA_{\su}[i]$ as $j_1 < \ldots < j_a$ (so $F(\child_{j_1})/N(\child_{j_1})\ge\cdots\ge F(\child_{j_a})/N(\child_{j_a})$). We now show how to compute $E_{i}(\su)$ at a given node, assuming that at each internal node we have access to $F$, $N$, $P_F$, $N_F$, $\cA_{\su}[i]$, and $Z_{\child_{j_s}}[i]$ for all $j_s\in\cA_{\su}[i]$. (We will show how to perform bottom-up computation of $\cA$ and $Z$ in Section~\ref{sec:bottom-up}.)

Pairs of children are of three types: unaffected-unaffected, affected-unaffected, affected-affected. Unaffected-unaffected pairs have zero externality since zeroing out $i$'s bids does not change their relative ordering. 

\paragraph{Affected-unaffected pairs.} Let $\child_{j_s}$ be an affected child whose average bid has decreased from $F(\child_{j_s})/N(\child_{j_s})$ to $G_i(\child_{j_s}) / N(\child_{j_s})$. An unaffected child before $\child_{j_s}$, that is, an unaffected child $\child_k$ with $k < j_s$, stays before $\child_{j_s}$. An unaffected child $\child_k$ with $k > j_s$ overtakes $\child_{j_s}$ if and only if $F(\child_k) / N(\child_k) > G_i(\child_{j_s}) / N(\child_{j_s})$.

Run binary search to find $$p_s\coloneq\max\left\{k\in\{j_s,\ldots,d\} : \frac{F(\child_k)}{N(\child_k)} > \frac{G_i(\child_{j_s})}{N(\child_{j_s})}\right\}.$$ So, all unaffected children in the interval $(j_s, p_s]$ overtake $\child_{j_s}$ in the counterfactual optimal ordering. Let $U_s\coloneq (j_s, p_s]\cap ([d]\setminus\cA_i(\su))$ denote this set of unaffected children. The externality due to all such unaffected children $\child_k, k\in U_s$, from overtaking $\child_{j_s}$ is $F(\child_k)N(\child_{j_s}) - G_i(\child_{j_s})N(\child_k)$, and so the total externality from all affected-unaffected child pairs is $$\sum_{s=1}^a\left(N(\child_{j_s})\sum_{k\in U_s}F(\child_k) - G_i(\child_{j_s})\sum_{k\in U_s}N(\child_k)\right).$$

We now show how to compute this quantity via prefix sums. For $s = 1,\ldots, a$ compute affected child prefix sums $$P_F^A[s]\coloneq\sum_{\ell=1}^s F(\child_{j_{\ell}});\quad P_N^A[s]\coloneq\sum_{\ell=1}^s N(\child_{j_{\ell}}).$$ 

Run binary search to compute $q_s\coloneq\max\{\ell : j_{\ell}\le p_s\}$. Then, $$\sum_{k\in U_s}F(\child_k) = \underbrace{\left(P_F[p_s] - P_F[j_s]\right)}_{\substack{\text{total bid under} \\ \text{children in }(j_s,p_s]}} - \underbrace{\left(P_F^A[r_s] - P_F^A[s]\right)}_{\substack{\text{total bid under} \\ \text{affected children in }(j_s,p_s]}}\eqcolon S_F[s]$$ and $$\sum_{k\in U_s}N(\child_k) = \underbrace{\left(P_N[p_s] - P_N[j_s]\right)}_{\substack{\text{total leaves under} \\ \text{children in }(j_s,p_s]}} - \underbrace{\left(P_N^A[r_s] - P_N^A[s]\right)}_{\substack{\text{total leaves under} \\ \text{affected children in }(j_s,p_s]}}\eqcolon S_N[s].$$ So, the total externality from affected-unaffected pairs is \begin{equation}\label{eq:AU-externality}\sum_{s=1}^a \left(N(\child_{j_s})S_F[s] - G_i(\child_{j_s})S_N[s]\right).\end{equation}

\paragraph{Affected-affected pairs.}

Sort the affected children in descending order of modified average bid. Let $\pi$ denote the sorting order, so $$\frac{G_i(\child_{j_{\pi(1)}})}{N(\child_{j_{\pi(1)}})}\ge\cdots\ge \frac{G_i(\child_{j_{\pi(a)}})}{N(\child_{j_{\pi(a)}})}.$$ The total externality from pairs of affected children is \begin{equation}\label{eq:AA-externality}\sum_{1\le s < t\le a}G_i(\child_{j_{\pi(s)}})N(\child_{j_{\pi(t)}}) - \sum_{1\le s < t\le a}G_i(\child_{j_{s}})N(\child_{j_{t}}).\end{equation}

For any ordering $\ell_1,\ldots,\ell_a$ of affected children, a sum of the form $\sum_{1\le s < t\le a}G_i(\child_{\ell_s})N(\child_{\ell_t})$ can be computed in $O(a)$ run-time: initialize $(R,C)\leftarrow (\sum_{t=1}^a N(\child_{\ell_t}), 0)$; for $s = 1,\ldots, a$ do $R\leftarrow R - N(\child_{\ell_s}), C\leftarrow C + G_i(\child_{\ell_s})R$; return $C$.

Finally, the total externality $E_i(\su)$ is the sum of expressions~\eqref{eq:AU-externality} and~\eqref{eq:AA-externality}.
We perform for each affected child one binary search on a list of size $O(d(\su))$ and another on a list of size $a$. Prefix sum computation of $P_F^A, P_N^A$ and affected-affected pair externality computation are $O(a)$ run-time computations. Sorting takes run-time $O(a\log a)$. So, the overall run-time for computing $E_i(\su)$ is $O(a\log d(\su))$.

\subsubsection{Computing all prices}\label{sec:bottom-up}

Let $\textsc{LocalExternality}(\su, \cA_u[i])$ return $E_i(\su)$, where $\cA_u[i] = [(j_1, Z_{\child_{j_1}}[i]),\ldots, (j_a, Z_{\child_{j_a}}[i])]$ is a list containing the children of $\su$ affected by $i$ and $Z_{\child_{j_s}}[i]$ is $i$'s aggregate bid under $\child_{j_s}$.

\begin{figure}[t]
\centering
\begin{minipage}[c]{0.42\linewidth}
\vspace{0pt}
\begin{algorithm}[H]
\SetAlgoLined
\SetAlgoNoEnd
Initialize $P_i\leftarrow 0$ for each user $i$\\
Let $\su$ be the root of the radix tree $\sT$\\
$\textsc{Process}(\su)$\\
\textbf{return} $[P_i/(n-1)]_{i=1}^m$
\caption{VCG Payments}
\label{alg:vcg_payments}
\end{algorithm}
\end{minipage}%
\hfill
\begin{minipage}[c]{0.52\linewidth}
\vspace{0pt}
\begin{algorithm}[H]
\SetAlgoLined
\SetAlgoNoEnd

\If{$\su$ \textup{is a leaf with request }$r\in\cR_i$}{
    \If{$b_r > 0$}{
    $Z_{\su} \leftarrow \{i\mapsto b_r\}$
    }
    \ElseIf{$b_r = 0$}{
    $Z_{\su}\leftarrow \emptyset$
    }
    \textbf{return} $Z_{\su}$
}

Initialize $Z_{\su}\leftarrow \emptyset, \cA_{\su}\leftarrow\emptyset$

\For{$j=1,\ldots,d(\su)$}{
    $Z_{\child_{j}}\leftarrow\textsc{Process}(\child_{j})$\\

    \For{$(i, z)\in Z_{\child_j}$}{
        $\cA_{\su}[i].\mathtt{append}((j, z))$\\
        $Z_{\su}[i]\leftarrow Z_{\su}[i] + z$
    }
}

\For{$i\in\cA_{\su}.\mathtt{keys}$}{

    $P_i\leftarrow P_i + \textsc{LocalExternality}(\su, \cA_{\su}[i])$
}
\textbf{return} $Z_{\su}$

\caption{$\textsc{Process}(\su)$}
\label{alg:process}
\end{algorithm}
\end{minipage}
\end{figure}

Algorithm~\ref{alg:vcg_payments} computes VCG payments via a bottom-up computation of affected children and aggregate bids (Algorithm~\ref{alg:process}). 

\subsubsection{Run-time analysis}

Global prefix sum arrays $P_F, P_N$ can be constructed in $O(\sum_{\su}d(\su))=O(n)$ run-time.

Let $\Gamma\coloneq\sum_{\su}\sum_i |A_{\su}[i]|$. Observe that
$
\Gamma
=
\sum_{\su}\sum_{j=1}^{d(\su)}|Z_{\child_j}|
=
\sum_{\su\neq \operatorname{root}(\sT)}|Z_{\su}|.$
Indeed, every entry $(i,z)\in Z_{\child_j}$ appends exactly one pair
$(j,z)$ to $\cA_{\su}[i]$, and every element of every affected-child
list arises in this way.

$\textsc{Process}$ called on the root of $\sT$ visits every node once, and every entry of the hash map $Z_{\child_j}$ is visited. Each call to $\textsc{LocalExternality}(u, \cA_{\su}[i])$ incurs $O(|\cA_{\su}[i]|\log d(\su))$ run-time. So, the total run-time across all calls to $\textsc{LocalExternality}$ is $O(\sum_{\su}\sum_{i}|\cA_{\su}[i]|\log d(\su)) \le O(\Gamma\log\nu)$.

The total run-time of $\textsc{Process}$ on the root of $\sT$ (Algorithm~\ref{alg:vcg_payments}) is therefore $O(n + \Gamma\log\nu)$.

We next bound $\Gamma$ in terms of the input size. For each non-root
node $\su$ and each entry $(i,z)\in Z_{\su}$, charge this entry to an
arbitrary positive-bid request $r\in\cR_i\cap\cL(\su)$. A request $r$ can receive
at most one charge from each non-root node on its root-to-leaf path, and
so its total charge is at most $\operatorname{depth}_{\sT}(r)$.
Consequently, $\Gamma
=
\sum_{\su\neq\operatorname{root}(\sT)}|Z_{\su}|
\le
\sum_{r=1}^n\operatorname{depth}_{\sT}(r)\le M + n = O(M).$

Also, $\Gamma = O(mn)$, since each map $Z_{\su}$ contains at most one entry per user and the
radix tree has $O(n)$ nodes. So, $\Gamma = O(\min\{M, mn\})$.

We have therefore proven Theorem~\ref{thm:payments}.

\thmprices*
\section{Appendix for Section~\ref{sec:autobidding}: Autobidding Agents for Inference Auctions}

\thmab*
\begin{proof}[Proof Sketch]
    The framework of~\citet{balseiro2020dual} requires (1) a convex and bounded action set that contains a ``do nothing'' action, (2) nonnegative, bounded, concave rewards, and (3) nonnegative, bounded resource consumption. 
    To obtain the required action set, we first subtract the priority received by a zero bid, which is independent of the autobidder’s report, and consider only additional priority. 
    For any fixed auction instance, the finitely many feasible schedules induce a finite menu of priority-and-payment outcomes. 
    Allowing randomization over this menu produces a bounded and convex action set with zero as a feasible option, as well as nonnegative and bounded expected utility. 
    Under VCG, the autobidder’s paced bid selects exactly the menu option that maximizes incremental utility after penalizing its payment by the current budget multiplier. 
    Thus, the two algorithms coincide until the remaining budget is no greater than the largest possible payment. 
    Finally, the randomized hindsight benchmark is at least as strong as our original deterministic benchmark in Definition~\ref{def:reg}.
\end{proof}
\begin{table}[htpb]
    \centering
    \begin{tabular}{p{2.6cm} c c}
        \toprule
         & Our Notation & \citet{balseiro2020dual}\\
        \midrule
        Time Horizon: & $T$ & $T$ \\
        Capacity: & $T\gamma$ & $T\rho$ \\
        Action Set: & $\mathcal{Y} =\{\vy\geq\mathbf{0} : \sum_k\vy[k]\leq 1\}$ & Generic action set $\mathcal{X}$ \\
        Decision\\ Variable: & Menu selection vector $y \in \mathcal{Y}$ & Action $x_t \in \mathcal{X}$ \\
        Reward\\ Function: & $\sum\limits_{k} (\langle\boldv^{(t)}_1,\va^{(t)}[k]\rangle -\vp^{(t)}[k])\vy[k]$ & $f_t(x)$ \\
        Resource\\ Consumption: & $\langle \vy, \vp^{(t)} \rangle$ & $b_tx_t$ \\
        Dual Variable: & $\mu^{(t)}$ & $\mu_t$ \\
        Primal\\ Objective:
        & $\max\limits_{\vy\in\mathcal{Y}} \sum\limits_{k} (
           \langle\boldv^{(t)}_1,\va^{(t)}[k]\rangle -(1+\mu^{(t)}) \vp^{(t)}[k] )\vy[k]$
        & $\max\limits_{x_t\in\mathcal{X}} \{f_t(x_t)-\mu_t^\top b_tx_t\}$ \\
        \bottomrule
    \end{tabular}
    \caption{Notation translation between our setting and~\citet{balseiro2020dual}.}
    \label{tab:detailed_notation_mapping}
\end{table}
\begin{proof}
Since bids, values, and prompts are all finite, $\cP$ will have finite support without loss of generality.

Define the shorthand $\vx^{(t)}(\vb_1) \coloneq \vx_1(\Bar{A}(\vb_1, \vb_{2:m}^{(t)}))$. 
Because the zero bid vector can receive positive priority for free, define $\vx^{(t,0)}\coloneq \vx^{(t)}(\mathbf{0})$ and $\va^{(t)}(\vb_1)\coloneq \vx^{(t)}(\vb_1) - \vx^{(t,0)}$. 
Thus $\va^{(t)}(\vb_1)$ is the vector of \emph{incremental} priorities obtained through bidding according to vector $\vb_1$ in round $t$. 
The baseline value $\sum_{t=1}^T \langle \boldv^{(t)}_1, \vx^{(t,0)} \rangle$ is independent of the autobidder’s reports and therefore does not affect either the optimal bidding policy or regret.

Fix a round $t$ and hold the other users' requests and bids fixed. 
As the autobidder varies its bid vector $\vb_1\in\mathbb{R}_+^{|\mathcal{R}_1^{(t)}|}$, the VCG mechanism presents it with a finite menu of incremental-priority vector and payment pairs 
\begin{equation*}
    \cM^{(t)} \coloneq \left\{ \left(\va^{(t)}(\vb_1), p^{(t)}(\vb_1) \right) : \vb\in\mathbb{R}_+^{|\mathcal{R}_1^{(t)}|} \right\}.
\end{equation*}
Because this menu is finite, we can vectorize its distinct elements as $\cM^{(t)} = \{(\va^{(t)}[k], \vp^{(t)}[k]) \}_{k=1}^{K^{(t)}}$, where $\va^{(t)}[k]$ is itself a vector of length $|\cR_1^{(t)}|$.\footnote{If $\langle \boldv^{(t)}_1, \va^{(t)}[k] \rangle - \vp^{(t)}[k] < 0$, we replace the item with $(\mathbf{0}, 0)$. This is without loss of generality since the bidder can always obtain zero incremental utility by bidding nothing.} 
Since every menu outcome is induced by a feasible DFS schedule, we have that $K^{(t)}\leq n!$. 
We can ``convexify'' the menu $\cM^{(t)}$ by introducing a vector $\vy \in \mathbb{R}_+^{n!}$ with $\sum_{k=1}^{n!} \vy[k] \leq 1$, where we interpret $\vy[k]$ as the probability of selecting menu item $k$.\footnote{If $K^{(t)} < n!$, we can pad $\vy$ with zeros at the end.} 
Let $\mathcal{Y}$ be the set of all such vectors $\vy$.  

Define $f^{(t)}(\vy) \coloneq  \sum_{k=1}^{n!} ( \langle \boldv^{(t)}_1, \va^{(t)}[k] \rangle - \vp^{(t)}[k]) \vy[k]$ and $\langle \vc^{(t)}, \vy \rangle \coloneq  \sum_{k=1}^{n!} \vp^{(t)}[k] \vy[k]$. 
This now is a valid instance of the online allocation framework of~\citet{balseiro2020dual} with one resource: 
The action set $\mathcal{Y}$ is convex, bounded, and contains the zero action. 
The reward function and resource consumption are both nonnegative and linear. 
Because the autobidder has at most $K$ requests and every request
value is at most $C$, $\langle \boldv^{(t)}_1,\va^{(t)}[k]\rangle \leq \|\boldv_1^{(t)}\|_1 \leq KC$. 
After replacing negative utility menu items with the zero item, every remaining menu item satisfies $0 \leq \langle \boldv^{(t)}_1,\va^{(t)}[k]\rangle-\vp^{(t)}[k] \leq KC$ and $0 \leq \vp^{(t)}[k] \leq \langle \boldv^{(t)}_1,\va^{(t)}[k]\rangle \leq KC$. 
Thus rewards and resource consumption are both bounded by $KC$, and the resource capacity is $T\gamma$.

Algorithm~\ref{alg:ogd} is the same as Algorithm 1 in~\citet{balseiro2020dual} specialized to online gradient descent and using our notation, up to its stopping time. 
To see this, fix a round $t$ and hold $\sT^{(t)}$ and $\vb_{2:m}^{(t)}$ fixed.
Let $W_{-1}^{(t)}(\sigma)$ denote the welfare of all players except the autobidder under $\sigma$, and let \(h^{(t)}_1 \coloneq  \max_{\sigma \in \cD(\sT^{(t)})}W_{-1}^{(t)}(\sigma)\). 
Under VCG, bidding according to bid vector $\vb_1$ selects 
\begin{equation*}
    \Bar{A}^{(t)}(\vb_1, \vb_{2:m}^{(t)})\in\arg\max_{\sigma \in \cD(\sT^{(t)})} \{ \langle \vb_1, \vx_1(\sigma) \rangle + W_{-1}^{(t)}(\sigma)\}
\end{equation*}
and incurs the user-level payment $p^{(t)}(\vb_{1},\vb_{2:m}^{(t)})=h^{(t)}_1 - W_{-1}^{(t)}(\Bar{A}^{(t)}(\vb_1, \vb_{2:m}^{(t)}))$. 
Substituting in the definitions of $\va^{(t)}$ and $p^{(t)}$ gives
\begin{equation*}
    \langle \vb_1,\va^{(t)}(\vb_1)\rangle- p^{(t)}(\vb_{1}, \vb_{2:m}^{(t)}) = \langle \vb_1,\vx^{(t)}(\vb_1)\rangle + W_{-1}^{(t)}(\Bar{A}^{(t)}(\vb_1, \vb_{2:m}^{(t)})) - (\langle \vb_1,\vx^{(t)}(\mathbf{0})\rangle + h_1^{(t)}).
\end{equation*}
For a fixed bid vector $\vb_1$, the final term on the right-hand side does not depend on the schedule that is selected. 
Therefore, maximizing $\langle \vb_1,\va^{(t)}(\vb_1)\rangle- p^{(t)}(\vb_1, \vb_{2:m}^{(t)})$ is equivalent to maximizing $\langle \vb_1,\vx_1(\sigma)\rangle + W_{-1}^{(t)}(\sigma)$. 
This is exactly reported total welfare, which is precisely what the VCG allocation rule maximizes. 

Using our notation, the primal step of the dual mirror descent algorithm of~\citet{balseiro2020dual} selects a menu item maximizing $\langle \boldv^{(t)}_1, \va \rangle - (1+\mu^{(t)})p$. 
The VCG mechanism implements this menu choice when the submitted bid vector is $\Tilde{\vb}^{(t)}= \frac{\boldv^{(t)}_1}{1 + \mu^{(t)}}$.

Let $\tau$ be the stopping time used by~\citet{balseiro2020dual}, namely the first round after which cumulative spending is within $CK$ of the total budget $T\gamma$. 
For every $t\leq \tau$, minimality of the stopping time implies that $B^{(t)}>CK$. 
Because $\|\alpha^{(t)} \boldv_1^{(t)}\|_1 \leq CK<B^{(t)}$ in every round $t \leq \tau$, both Algorithm~\ref{alg:ogd} and Algorithm 1 in~\citet{balseiro2020dual} therefore execute the menu item induced by $\alpha^{(t)} \boldv_1^{(t)}$, incur the same payment, and perform the same dual update in every round through $\tau$. 
The proof of Theorem 1 in~\citet{balseiro2020dual} bounds regret by the primal-dual gap accumulated through $\tau$, together with a term $CK\mathbb{E}[T-\tau]$ accounting for the remaining rounds. 
The former depends only on the algorithms’ behavior through $\tau$, while their stopping-time bound controls the latter. 
Since rewards are nonnegative, the rewards earned after $\tau$ may be discarded when lower-bounding Algorithm~\ref{alg:ogd}’s total reward. Thus the algorithms need not coincide after $\tau$.
All assumptions for Theorem 1 in~\citet{balseiro2020dual} now hold, and so translating it into our notation gives the desired regret guarantee. 
Finally, it is worth noting that the guarantees of~\citet{balseiro2020dual} are with respect to an adversary optimizing over $\mathcal{Y}$. 
However the convexified hindsight benchmark can only be larger than the deterministic hindsight optimum, so this bound holds in our setting as well.\looseness-1 

The regret bound in~\citet{balseiro2020dual} give us the guarantee 
\begin{equation*}
    R(T) \leq 2(C^2 K^2 + \gamma^2)\eta T + \frac{C K}{\gamma\eta} \left(\frac{C K}{\gamma}+1\right) + \frac{C^2 K^2}{\gamma}.
\end{equation*}
Setting $\eta = \sqrt{\frac{\frac{CK}{\gamma}(\frac{CK}{\gamma} + 1)}{2T(C^2K^2 + \gamma^2)}}$ gives us a regret bound of 
\begin{equation*}
    R(T) \leq 2\sqrt{\frac{CK}{\gamma} (\frac{CK}{\gamma} + 1) \cdot 2 (C^2 K^2 + \gamma^2) T} + \frac{C^2 K^2}{\gamma},
\end{equation*}
as long as $T \geq \frac{C^3 K^3}{2\gamma(C^2 K^2 + \gamma^2)}(\frac{CK}{\gamma} + 1)$.
    
\end{proof}

\begin{lemma}[DFS Welfare Decomposition]
    If a DFS schedule $\sigma$ visits the children $\sfc_1, \ldots, \sfc_{d(\su)}$ of each internal node $\su$ in that order, then 
    \begin{equation}\label{eq:decomp1}
        (n-1) W(\sigma; \boldv) = \sum_{\su} \sum_{1 \leq j < k \leq d(\su)} N(\sfc_j) N(\sfc_k) \Bar{v}(\sfc_j).
    \end{equation}
    Consequently, 
    \begin{equation}\label{eq:decomp2}
        (n-1) \max_{\sigma \in \cD(T)} W(\sigma; \boldv) = \sum_{\su} \sum_{\{\sfc,\sd\} \subseteq \cC(\su)} N(\sfc) N(\sd) \max\{\Bar{v}(\sfc), \Bar{v}(\sd)\}.
    \end{equation}
\end{lemma}
\begin{proof} 
    Both equations follow from the same reasoning as in the proof of Theorem~\ref{thm:welfare}.
\end{proof}

\thmcondition*
\begin{proof}
    Fix an internal node $\su$ and an unordered pair of children $\{\sfc, \sd\}$. 
    Without loss of generality suppose that $\Bar{v}(\sfc) \geq \Bar{v}(\sd)$. 
    The maximum welfare contribution of this pair is therefore $N(\sfc) N(\sd) \Bar{v}(\sfc)$. 
    We now compare this with the contribution under the average-bid order:
    \begin{enumerate}
        \item If $\Bar{A}(\vb_{1:m})$ places $\sfc$ before $\sd$, then the pair contributes exactly $N(\sfc) N(\sd) \Bar{v}(\sfc)$, and so it obtains its maximum possible contribution.
        \item Now suppose that $\Bar{A}(\vb_{1:m})$ places $\sd$ before $\sfc$. 
        Because the children are ordered by descending average bid, $\Bar{b}(\sd) \geq \Bar{b}(\sfc)$. 
        By Definition~\ref{def:subtree_pacing}, we have that $\Bar{b}(\sfc) = \alpha(\sfc) \Bar{v}(\sfc)$ and $\Bar{b}(\sd) = \alpha(\sd) \Bar{v}(\sd)$, and so $\alpha(\sd) \Bar{v}(\sd) \geq \alpha(\sfc) \Bar{v}(\sfc)$. 
        Rearranging terms, we have that 
        \begin{equation*}
            \frac{\Bar{v}(\sd)}{\Bar{v}(\sfc)} \geq \frac{\alpha(\sfc)}{\alpha(\sd)} \geq \frac{1}{\kappa_{\su}}
        \end{equation*}
        and therefore the pair's welfare contribution is $N(\sfc) N(\sd) \Bar{v}(\sd) \geq \frac{1}{\kappa_{\su}} N(\sfc) N(\sd) \Bar{v}(\sfc)$. 
    \end{enumerate}

    Therefore, whether or not the average-bid order agrees with the average value order, every unordered pair of children of $\su$ obtains at least $1/\kappa_{\su}$ of its maximum possible welfare contribution. 
    Summing over all unordered pairs and internal nodes, we get that 
    \[\begin{aligned}
        (n-1) W(\sigma; \boldv_{1:m}) &= \sum_{\su} \sum_{\sfc,\sd \in \cC(\su), \sfc \prec_{\Bar{A}(\vb_{1:m})} \sd} N(\sfc) N(\sd) \Bar{v}(\sfc) \\ &\geq \sum_{\su} \frac{1}{\kappa_{\su}} \sum_{\{\sfc, \sd\} \subseteq \cC(\su)} N(\sfc) N(\sd) \max\{\Bar{v}(\sfc), \Bar{v}(\sd)\}
    \end{aligned}\]
    by Equation~\ref{eq:decomp1}. 
    To obtain the final result, we can apply Equation~\ref{eq:decomp2} to get
    \begin{equation*}
        \frac{W(\Bar{A}(\vb_{1:m}); \boldv_{1:m})}{\max_{\sigma \in \cD(T)} W(\sigma; \boldv_{1:m})} \geq \frac{\sum_{\su} \frac{1}{\kappa_{\su}} \sum_{\{\sfc, \sd\} \subseteq \cC(\su)} N(\sfc) N(\sd) \max\{\Bar{v}(\sfc), \Bar{v}(\sd)\}}{\sum_{\su} \sum_{\{\sfc,\sd\} \subseteq \cC(\su)} N(\sfc) N(\sd) \max\{\Bar{v}(\sfc), \Bar{v}(\sd)\}} \geq \frac{1}{\max_{\su} \kappa_{\su}}.
    \end{equation*}
    where $\Omega_{\su} \coloneq \frac{1}{n-1} \sum_{\{\sfc, \sd\} \subseteq \cC(\su)} N(\sfc) N(\sd) \max\{F_v(\sfc) / N(\sfc), F_v(\sd) / N(\sd)\}$. 
\end{proof}

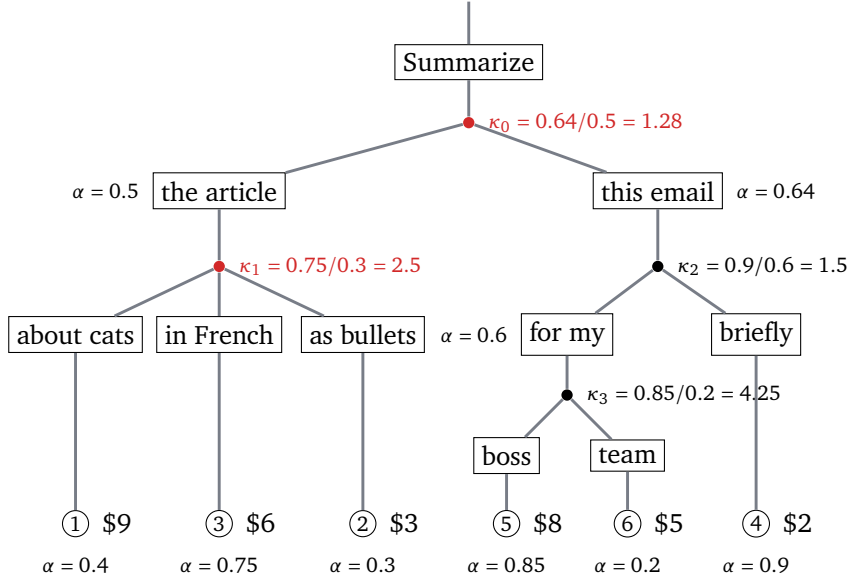
\begin{figure}[t]
\centering
\definecolor{dfsgray}{RGB}{120,125,135}
\definecolor{flipred}{RGB}{210,40,40}
\begin{tikzpicture}[
  tok/.style={draw, fill=white, inner sep=3pt, font=\small},
  br/.style={circle, fill, inner sep=1.5pt},
  leaf/.style={circle, draw, inner sep=1pt, font=\scriptsize},
  val/.style={anchor=west, xshift=1pt, font=\small},
  alph/.style={font=\scriptsize},
  kap/.style={anchor=west, xshift=4pt, font=\scriptsize},
  edge/.style={line width=1.1pt, draw=dfsgray},
]
  \node[tok] (s)  at (0,0)    {Summarize};
  \node[br, flipred] (d0) at (0,-0.8) {};
  \draw[edge] (0,0.8) -- (s) -- (d0);
  \node[kap, text=flipred] at (d0) {$\kappa_0 = 0.64/0.5 = 1.28$};

  \node[tok] (art) at (-3.3,-1.7) {the article};
  \node[tok] (em)  at ( 2.5,-1.7) {this email};
  \node[alph, anchor=east] at (art.west) {$\alpha = 0.5$\,};
  \node[alph, anchor=west] at (em.east)  {\,$\alpha = 0.64$};
  \node[br, flipred] (d1) at (-3.3,-2.7) {};
  \node[br]          (d2) at ( 2.5,-2.7) {};
  \draw[edge] (d0) -- (art) -- (d1);
  \draw[edge] (d0) -- (em)  -- (d2);
  \node[kap, text=flipred] at (d1) {$\kappa_1 = 0.75/0.3 = 2.5$};
  \node[kap] at (d2) {$\kappa_2 = 0.9/0.6 = 1.5$};

  \node[tok] (cats)  at (-5.2,-3.6) {about cats};
  \node[tok] (fr)    at (-3.3,-3.6) {in French};
  \node[tok] (bul)   at (-1.4,-3.6) {as bullets};
  \node[tok] (formy) at ( 1.3,-3.6) {for my};
  \node[tok] (brf)   at ( 3.8,-3.6) {briefly};
  \node[alph, anchor=east] at (formy.west) {$\alpha = 0.6$\,};
  \draw[edge] (d1) -- (cats);
  \draw[edge] (d1) -- (fr);
  \draw[edge] (d1) -- (bul);
  \draw[edge] (d2) -- (formy);
  \draw[edge] (d2) -- (brf);

  \node[br] (d3) at (1.3,-4.4) {};
  \draw[edge] (formy) -- (d3);
  \node[kap] at (d3) {$\kappa_3 = 0.85/0.2 = 4.25$};
  \node[tok] (boss) at (0.5,-5.2) {boss};
  \node[tok] (team) at (2.1,-5.2) {team};
  \draw[edge] (d3) -- (boss);
  \draw[edge] (d3) -- (team);

  \foreach \r/\x/\v/\a/\src in {1/-5.2/9/0.4/cats, 3/-3.3/6/0.75/fr,
                                2/-1.4/3/0.3/bul, 5/0.5/8/0.85/boss,
                                6/2.1/5/0.2/team, 4/3.8/2/0.9/brf} {
    \node[leaf] (l\r) at (\x,-6.1) {\r};
    \node[val] at (l\r.east) {\$\v};
    \node[alph] at (\x,-6.65) {$\alpha = \a$};
    \draw[edge] (\src) -- (l\r);
  }
\end{tikzpicture}
\caption{Welfare under heterogeneous pacing. Leaves show true values $v_r$
and pacing factors $\alpha_r$, so that $b_r = \alpha_r v_r$. Internal
subtrees are labeled with their value-weighted pacing factor
$\alpha(\sfc) = F_b(\sfc)/F_v(\sfc)$, and $\kappa_{\su}$ is the ratio of the largest to
smallest $\alpha(\sfc)$ among the children of $\su$. Red nodes are where paced
bids reverse the value-optimal order. At the root, ``this email'' now has
the higher average bid ($3.2$ versus $3$), and under ``the article,''
request~3 now outbids request~1 ($4.5$ versus $3.6$). The realized schedule
is $5, 6, 4, 3, 1, 2$, which attains $87/99 \approx 0.88$ of the optimal
DFS welfare. Theorem~\ref{thm:condition} guarantees at least
$\sum_{\su} \Omega_{\su}/\kappa_{\su} \big/ \sum_{\su} \Omega_{\su} \approx 0.63$, and hence at
least $1/\max_{\su} \kappa_{\su} \approx 0.24$. Node~3 has the largest condition
number but no reversal, so it loses nothing.}
\label{fig:pacing}
\end{figure}
\section{Appendix for Section~\ref{sec:expts}: Experiments}\label{app:expts}

\subsection{Setup Details}\label{app:protocol}

We use SGLang v0.5.17\footnote{\url{https://github.com/sgl-project/sglang}} serving DeepSeek-R1-Distill-Qwen-32B in bf16 with
the engine's defaults, except that we fix its static memory fraction at
$0.85$, which leaves a 23k-token KV cache. Prefix caching and LRU
eviction are handled by SGLang. Every query decodes at most 128
tokens, so TTFT is dominated by prefill and caching. The queue in front
of the server holds $Q$ queries, splits them into $Q/n$ batches, orders
each batch with one of our algorithms,
and sends them sequentially without pausing the
engine. The cache is cleared once per experiment (not between batches).
The queue sends about 20 queries per second and the engine serves 3 to
4, so the engine's own queue has a few hundred entries for the whole run;
TTFT is almost entirely time spent waiting in the queue (in the order we
sent). The queue tokenizes prompts with the model's tokenizer, so the
tree it optimizes is the same one built by SGLang.

We ran the experiments on three H100 nodes with slightly varying hardware.
Welfare and hit rate do not depend on the machine; TTFT does, so throughout the paper we
only compare TTFTs measured on the same machine. The threshold,
correlation, and long-tree experiments are from the second machine and
the one-cluster experiment from the third; everything else is from the first.

Welfare is measured with respect to the final ordering of queries ranked by TTFT. It agrees with the welfare of the planned average-bid DFS order to within $1.3\%$ for
all results, so we report only the realized value.
Raw welfare varies substantially from trial to trial because a heavy-tailed
draw puts a different amount of value in the batch each time, and it
moves every algorithm together (Figure~\ref{fig:raw}), which is why we
plot the percentage over $\tfrac12\sum_r b_r$. Cache hit rate is defined as the number of cached
prompt tokens divided by the total number of prompt tokens, as reported by SGLang. All completion times, when reported, record the time from the first send to the last
completion of the batch.
The queue itself is a Python prototype: for a
batch of $1024$ queries it computes the order in under a millisecond
and the VCG payments in $0.3$\,s by naive per-user recomputation
(Appendix~\ref{app:fastvcg} has the fast version), while actually
serving the batch takes on the order of minutes.

\subsection{Workloads}\label{app:workloads}

\mixsharename\ mixes three kinds of queries: SQuAD questions, 32 per
article, with the article as a shared prefix; GSM8K questions over one
of 12 few-shot prefixes; and singletons that share only a short
instruction. In \mixshare{P}, $P\%$ of the queries belong
to one of the shared prefix clusters. \agent\ replays logged SWE-agent trajectories, 
where each call's prompt
is the history up to that call, so call $k$ is a prefix of call $k+1$.
\tot\ replays tree-of-thought on GSM8K (3-ary thought trees of depth 3).
\totlong\ is the same, with a SQuAD article of 3-5k tokens as the root of every tree;
since each thought is a single sentence, nearly every token of every
prompt is the article, which is why its shared prefixes are so long and
its unique suffixes so short.
Table~\ref{tab:workloads} gives the tree statistics of one $Q=1024$
instance of each and Figure~\ref{fig:workloads} the per-query
distributions. Note the third column of the table:
every workload but \tot\ exceeds the KV cache several times over.

\begin{table}[htbp]

\centering\scriptsize
\begin{tabular}{lrrrrrrrr}
\toprule
 & total & distinct & distinct $/$ & shared & \multicolumn{3}{c}{median tokens per query} & queries at \\
workload & tokens & tokens & KV cache & fraction & prompt & shared prefix & unique suffix & interior nodes \\
\midrule
stars & 2.05M & 781k & 33.7 & 0.62 & 1086 & 381 & 57 & 0 \\
chains      & 3.32M & 437k & 18.9 & 0.87 & 2854 & 2852 & 0 & 876 \\
short trees   & 168k  & 23k  & 1.0  & 0.86 & 159  & 147  & 12 & 357 \\
long trees & 4.07M & 124k & 5.4 & 0.97 & 4143 & 4127 & 15 & 335 \\
\bottomrule
\end{tabular}

\caption{Radix-tree statistics of one $Q=1024$ instance of each
workload. Distinct tokens is the total edge length of the tree; a
query's shared prefix is its longest common prefix with any other
query. Stars corresponds to \mixshare{50}, chains corresponds to \agent, short trees corresponds to \tot, and long trees corresponds to \totlong.}
\label{tab:workloads}
\end{table}

\begin{figure}[htbp]
    \centering
    \includegraphics[width=0.92\linewidth]{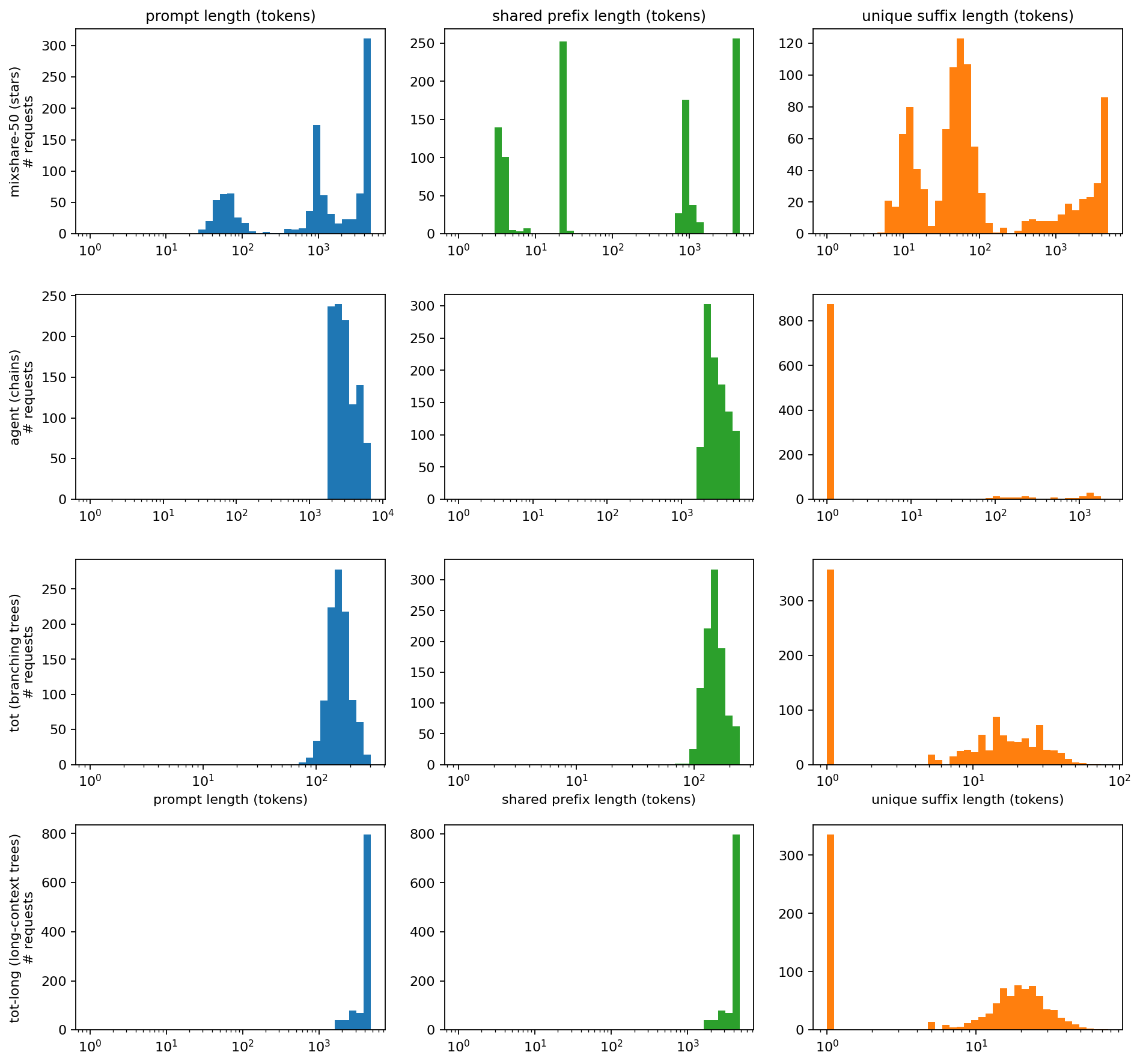}
    \caption{Per-query distributions of prompt, shared prefix, and
    unique suffix length (log scale).}
    \label{fig:workloads}
\end{figure}

\subsection{Batch Size and Sharing Fraction}\label{app:stars}

Figure~\ref{fig:raw} gives the raw welfare and TTFT behind the first
two columns of Figure~\ref{fig:grid}. In the sharing sweep at $n=256$,
nothing can be cached at $P=0$ and the algorithms differ only through
the order itself. As $P$ grows, the completion time of the two DFS orders falls
from about 560\,s to 330\,s while that of \bidsort\ and \nosort\ rises
to 740\,s. In the batch-size sweep on \mixshare{50}, \nosort\ does not
depend on $n$ by construction, which is a useful consistency check;
for \dfsbid, the hit rate goes from $0.23$ at $n=32$ to $0.60$ at $n=Q$
and mean TTFT from 316\,s to 160\,s. The ratio of the best DFS welfare
to the unconstrained optimum falls gently with sharing, from $0.83$ at
$P=0$ to $0.78$ at $P=75\%$, as deeper trees hit the DFS restriction
more often. While one might expect \dfsbid\ and \bidsort\ to agree at $P=0$,
singletons in our dataset contain a very short shared instruction,
causing us to serve one group before another.

\begin{figure}[htbp]
    \centering
    \includegraphics[width=0.8\linewidth]{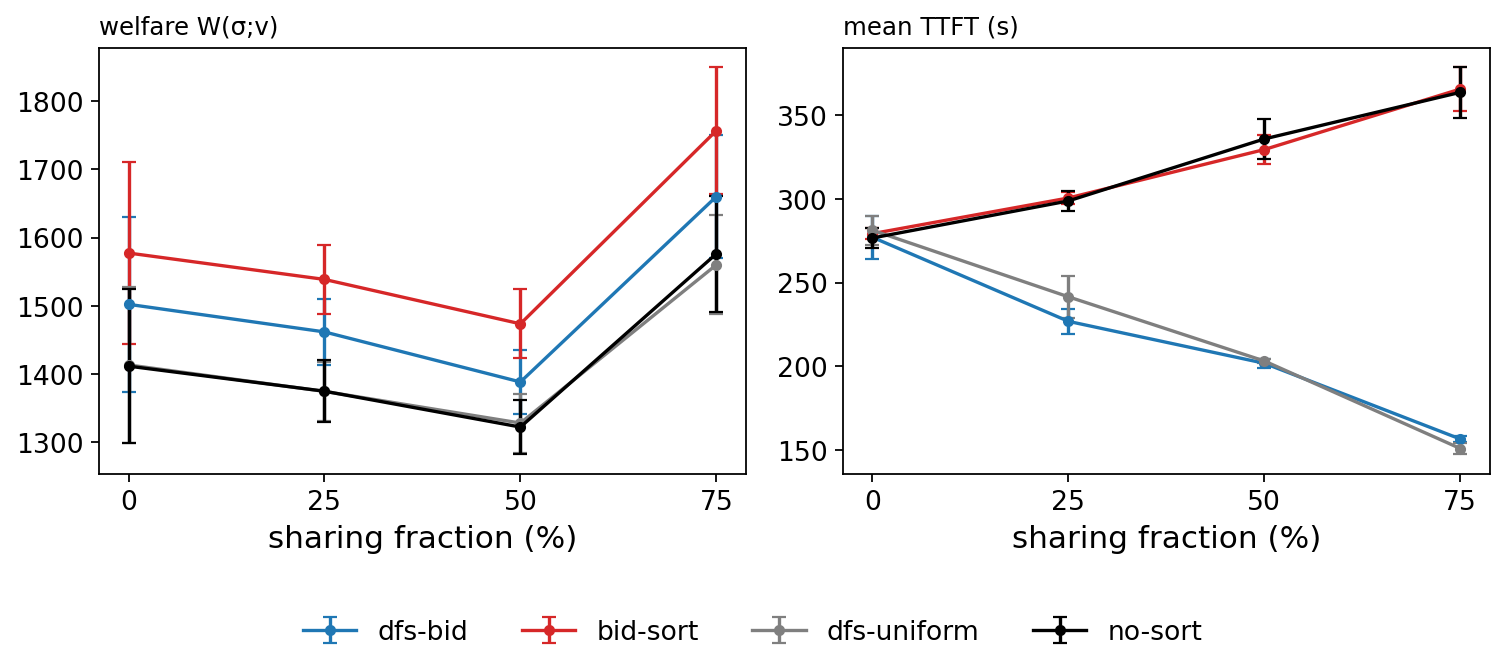}\\[2pt]
    \includegraphics[width=0.8\linewidth]{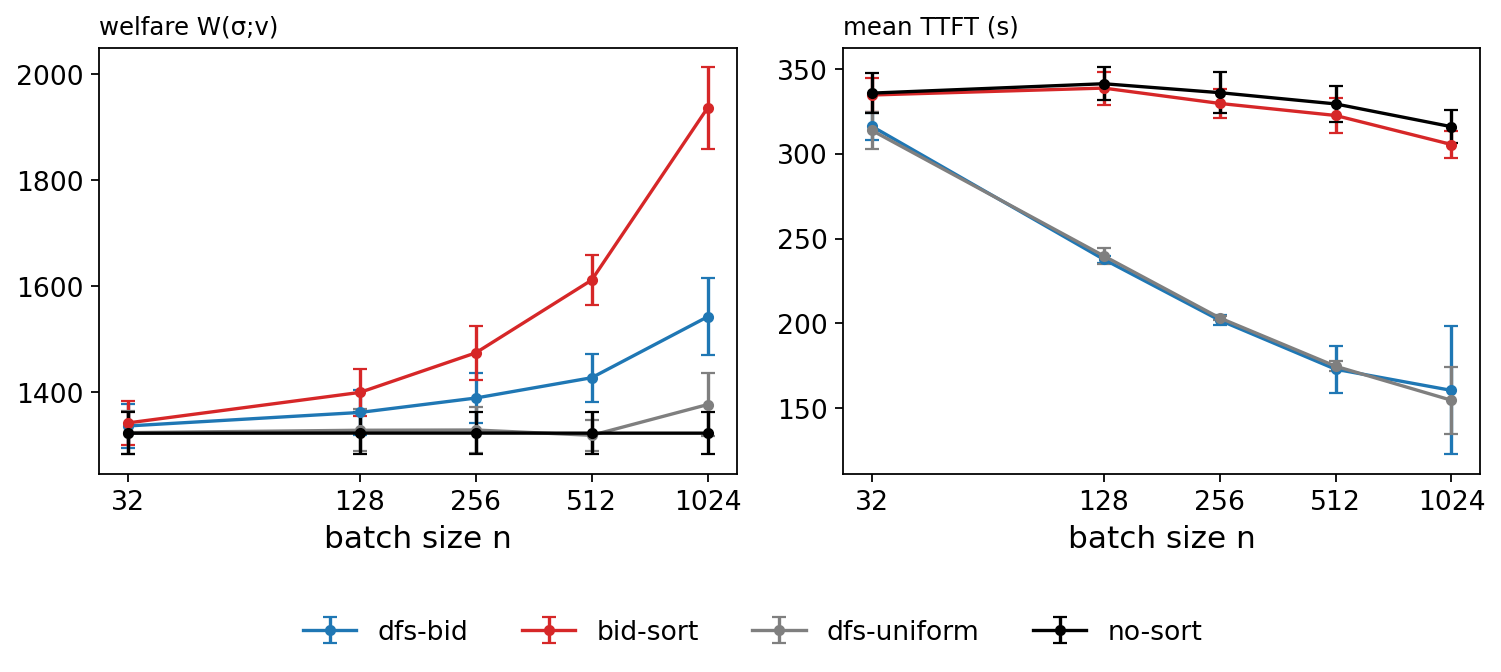}
    \caption{Raw welfare and mean TTFT for the sharing-fraction (top)
    and batch-size (bottom) sweeps of Figure~\ref{fig:grid}.}
    \label{fig:raw}
\end{figure}

\subsection{Different Tree Topologies}\label{app:shapes}

Figures~\ref{fig:agent}, \ref{fig:tot}, and~\ref{fig:totlong} vary
the batch size on different radix tree topologies: chains, short trees, and long trees. On chains at
$n=Q$, \dfsbid\ beats \dfsuniform\ by $16\%$ at the same hit rate and
TTFT, while \bidsort\ wins welfare by $48\%$ and pays with a $0.48$ hit
rate (versus \dfsbid's $0.75$) and $2.4\times$ the TTFT. \nosort\ and \bidsort\
have the same hit rate at every $n$, since sorting a random order by
bid leaves it random from the cache's point of view. Short trees fit
in the cache, so nothing is evicted. Here, every algorithm has the same 28\,s
TTFT. 
Long trees are where the DFS restriction of our inference auction has the largest impact. \dfsbid\ keeps a $0.96$ hit rate and a 46\,s TTFT, and \bidsort\
drops to $0.20$ and 559\,s, because every branch it interleaves
recomputes a 4k-token prefix.

\begin{figure}[htbp]
    \centering
    \includegraphics[width=\linewidth]{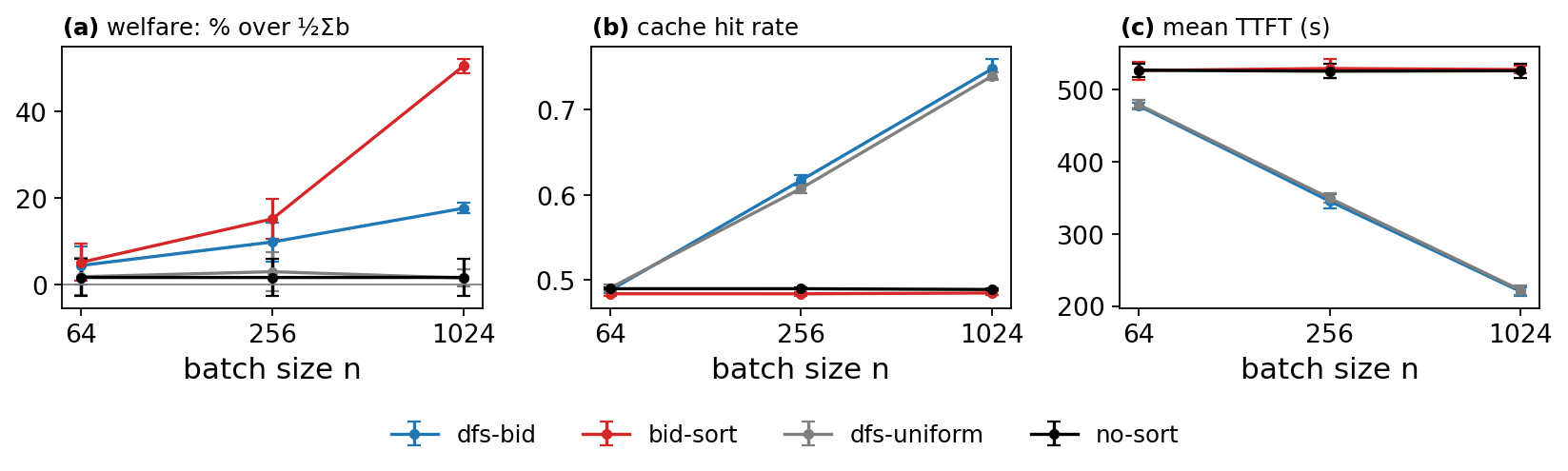}
    \caption{Chains (\agent) against batch size $n$.}
    \label{fig:agent}
\end{figure}

\begin{figure}[htbp]
    \centering
    \includegraphics[width=\linewidth]{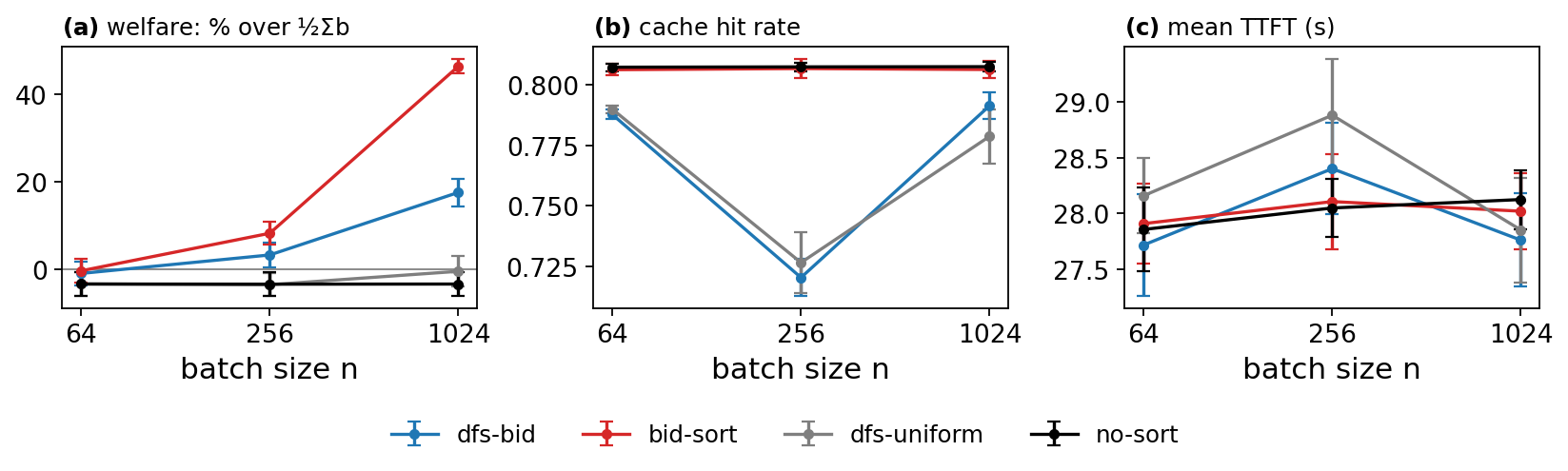}
    \caption{Short trees (\tot) against batch size $n$. The workload
    fits in the cache, so hit rate and TTFT do not depend on the order.}
    \label{fig:tot}
\end{figure}

\begin{figure}[htbp]
    \centering
    \includegraphics[width=\linewidth]{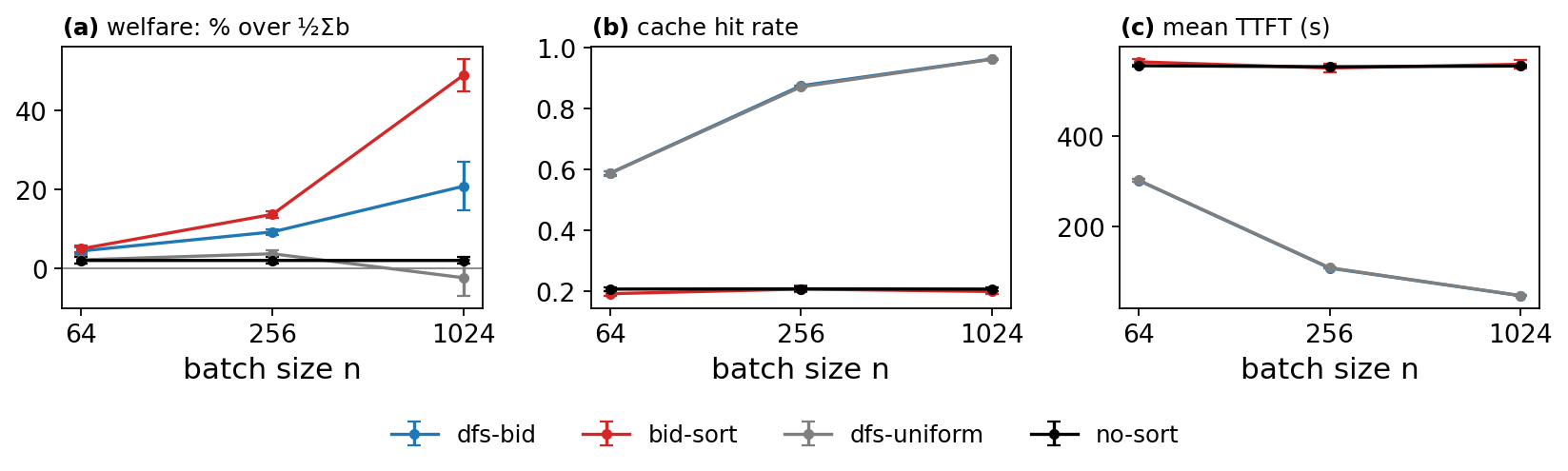}
    \caption{Long trees (\totlong) against batch size $n$, second
    machine.}
    \label{fig:totlong}
\end{figure}

\subsection{Correlated Values}\label{app:corr}

The main results draw bids i.i.d.\ per query, which is the case of
many users sharing a prefix with unrelated values and the hardest case
for \dfsbid. To interpolate toward the case where one user owns a whole
cluster, we draw values from a Gaussian copula with within-cluster
correlation $\rho$ and the same Pareto marginal. A cluster is a set of
queries sharing at least 128 prefix tokens on stars, a trajectory on
chains, and a tree on trees; singletons are unaffected. Each cluster $c$
draws $Z_c\sim\mathcal{N}(0,1)$ and each query $r$ in it draws
$\epsilon_r\sim\mathcal{N}(0,1)$; its value is
$v_r = 1 + (1-\Phi(z_r))^{-1/1.5}$ with
$z_r=\sqrt{\rho}\,Z_c+\sqrt{1-\rho}\,\epsilon_r$. So $\rho=0$ is i.i.d.\
and $\rho=1$ gives every query in a cluster the same value.

Figure~\ref{fig:offline} plots welfare against $\rho$
on stars and the right column of Figure~\ref{fig:grid} displays the engine
results at $\rho\in\{0,0.5,1\}$. \dfsbid\ gains as $\rho$ grows because
the value-sorted order and the tree-contiguous order agree more often:
the best DFS welfare rises from $0.83$ to $0.91$ of the unconstrained
optimum. \bidsort\ does not gain, and it still thrashes the cache at
$\rho=0.5$ (hit rate $0.29$ against $0.60$). At $\rho=1$ its order
happens to keep clusters together and its hit rate matches \dfsbid's,
but \dfsbid\ is still faster (mean TTFT 85\,s against 124\,s), because
it serves short high-value clusters first.

Figure~\ref{fig:corrshapes} repeats this experiment on chains and long trees
at $n=Q$. On long trees at $\rho=1$ the two optima coincide: \dfsbid\
reaches $0.999$ of the unconstrained welfare at the same hit rate and
TTFT as \bidsort. On chains it reaches $0.90$, and the remaining gap is
the cost of the DFS restriction when trajectories on the same task
share interior nodes with other users. On both, \bidsort\ recovers the
cache only at $\rho=1$; at $\rho=0.5$ its TTFT is still $2\times$
\dfsbid's on chains and $10\times$ on long trees.

\begin{figure}[htbp]
    \centering
    \includegraphics[width=\linewidth]{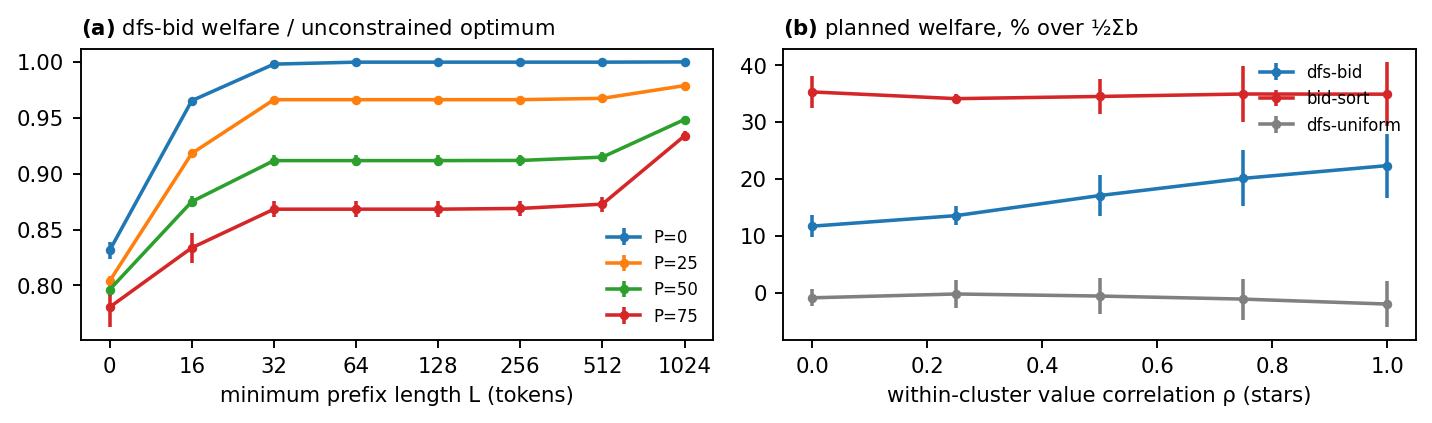}
    \caption{Offline sweeps on the planned orders, three trials. (a)
    \dfsbid's welfare as a fraction of the unconstrained optimum
    against the minimum-prefix threshold $L$, per sharing fraction. (b) Planned welfare
    of each algorithm against within-cluster value correlation $\rho$.}
    \label{fig:offline}
\end{figure}

\begin{figure}[htbp]
    \centering
    \includegraphics[width=\linewidth]{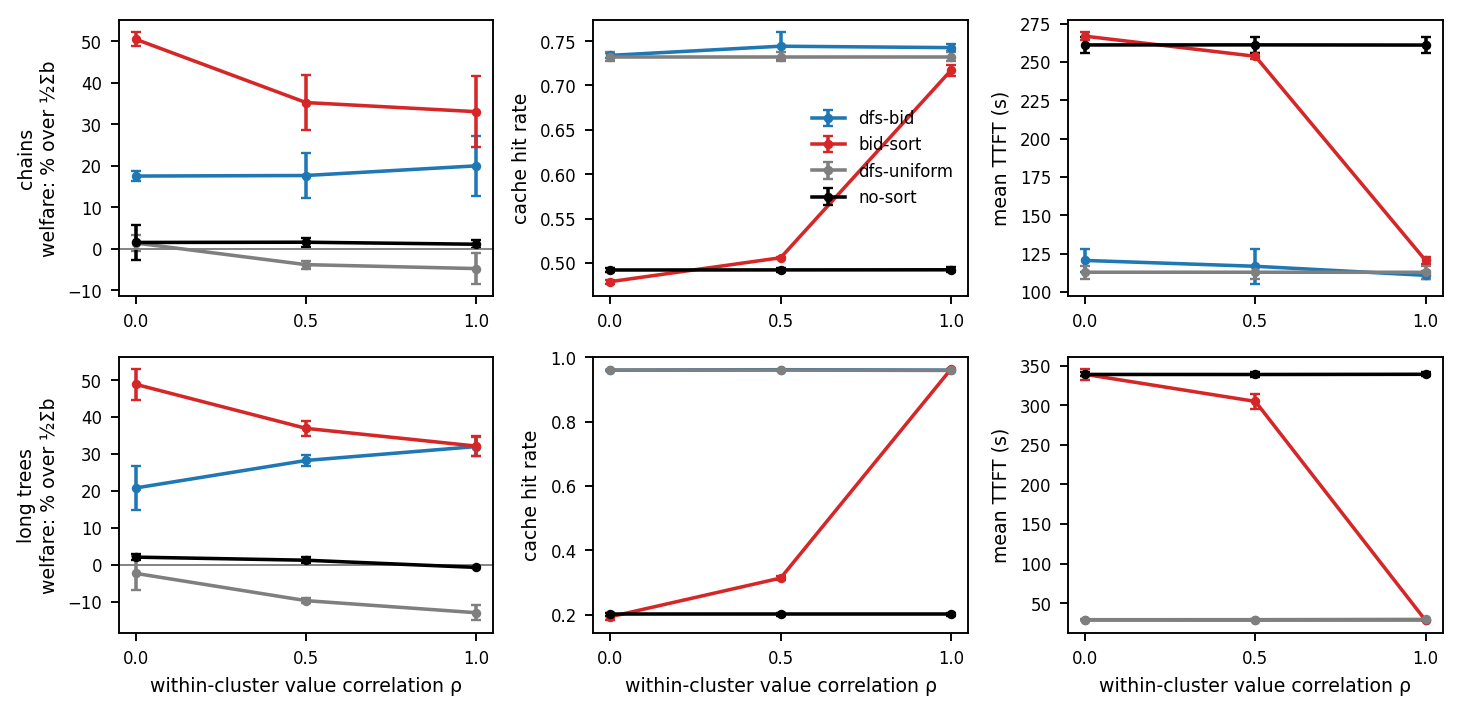}
    \caption{Within-cluster value correlation $\rho$ on chains (top) and
    long trees (bottom) at $n=Q$, third machine.}
    \label{fig:corrshapes}
\end{figure}

\subsection{One Cluster Bids}\label{app:juice}

We consider now a situation similar to that of the highest-correlation setting of the experiments in Figures~\ref{fig:offline} and~\ref{fig:corrshapes} (though in the original setting of i.i.d. Pareto bids).
We draw one cluster of queries with a shared
prefix at random among the clusters with at least eight members (32 to
58 queries on these instances). That cluster keeps its Pareto bids. Every
other bid is set to zero (Figure~\ref{fig:juice}). 

This bid distribution is meant to model a realistic scenario where most users do not submit any bid for priority (such users still participate in the auction and must be served in a timely fashion).

Serving the cluster with non-zero bids first corresponds to both the unconstrained
optimum and a DFS order, so \dfsbid\ matches \bidsort's welfare at every
where the paying cluster gets its first tokens in about a
second under either algorithm. The margin over the bid-blind algorithms
grows with $n$, from nothing at $n=32$, where a batch rarely holds more
than one or two members of the cluster, to about $100$ points at $n=Q$.
\dfsbid's hit rate and completion time equal \dfsuniform's at
every $n$, while \bidsort's hit rate stays near $0.25$ and its TTFT is
about $1.5\times$ at every $n$ above $32$.\looseness-1

\subsection{Online Arrivals}\label{app:online}

The main protocol has the whole queue present at $t=0$, so the
batch-size sweep in Figure~\ref{fig:grid} shows the benefit of a larger
batch but not its cost (the wait for the batch to fill). Here the $1024$
queries of \mixshare{50} arrive as a Poisson process at 4 or 1 per
second, a batch is released when its $n$-th query arrives, and TTFT is
measured from arrival. \nosort\ is batched the same way, so the batching
cost is common to all algorithms. Table~\ref{tab:online} reports means
over three trials.

\begin{table}[htbp]
\centering\small
\begin{tabular}{llrrrr}
\toprule
rate & $n$ & wait & \multicolumn{2}{c}{\dfsbid} & \nosort \\
 & & & mean TTFT & hit & mean TTFT \\
\midrule
4/s & 64  & 8   & 150 & 0.33 & 211 \\
4/s & 512 & 63  & 160 & 0.56 & 314 \\
1/s & 64  & 31  & 45  & 0.30 & 47 \\
1/s & 512 & 254 & 328 & 0.56 & 406 \\
\bottomrule
\end{tabular}
\caption{Poisson arrivals on \mixshare{50}. ``Wait'' is the mean time
from a query's arrival to the release of its batch; all times in
seconds.}
\label{tab:online}
\end{table}

At 4 queries per second the server is saturated: the queue arrives in
about 250\,s and takes over 500\,s to serve. Going from $n=64$ to
$n=512$ adds 55\,s of batch wait but changes \dfsbid's mean TTFT by
between $-7$ and $+29$\,s across trials, because the added wait mostly
replaces time the query would have spent waiting in the engine, while
the hit rate rises from $0.33$ to $0.56$. The same change costs
\nosort\ 100\,s and buys it nothing. At 1 query per second the server
keeps up with arrivals, there is no queue to hide the wait, and $n=512$
costs about 280\,s of TTFT. Under load a large batch is nearly free;
under light load a timeout should release batches long before they
fill.

\subsection{Second Model}\label{app:model2}

Table~\ref{tab:model2} repeats our main experiments with
Qwen2.5-7B-Instruct~\citep{qwen2025qwen25} in place of the 32B model,
on the same instances and with the KV cache capped at the same number
of tokens. The two models share a tokenizer, so the trees are
identical. Welfare agrees to the unit in every experiment and hit rates
are very similar; time metrics are smaller
due to the smaller size of the model.\looseness-1

\begin{table}[htbp]
\centering\small
\begin{tabular}{llrrrrrr}
\toprule
 & & \multicolumn{2}{c}{welfare} & \multicolumn{2}{c}{hit rate} & \multicolumn{2}{c}{completion time (s)} \\
experiment & algorithm & 7B & 32B & 7B & 32B & 7B & 32B \\
\midrule
\mixshare{50}, $n=Q$ & \dfsbid  & 1541 & 1542 & 0.60 & 0.60 & 87 & 319 \\
                     & \bidsort & 1936 & 1936 & 0.22 & 0.23 & 163 & 615 \\
                     & \dfsuniform & 1376 & 1376 & 0.60 & 0.59 & 91 & 322 \\
                     & \nosort  & 1322 & 1322 & 0.21 & 0.21 & 165 & 629 \\
\mixshare{75}, $n=256$ & \dfsbid & 1660 & 1660 & 0.70 & 0.69 & 81 & 327 \\
                     & \bidsort & 1757 & 1757 & 0.29 & 0.28 & 188 & 739 \\
\agent, $n=Q$        & \dfsbid  & 1853 & 1854 & 0.82 & 0.75 & 77 & 438 \\
                     & \bidsort & 2375 & 2375 & 0.49 & 0.48 & 227 & 1076 \\
\tot, $n=Q$          & \dfsbid  & 1686 & 1686 & 0.86 & 0.79 & 54 & 61 \\
                     & \bidsort & 2098 & 2098 & 0.78 & 0.81 & 53 & 60 \\
\bottomrule
\end{tabular}
\caption{Qwen2.5-7B-Instruct against the 32B model on identical
instances; means over three trials.}
\label{tab:model2}
\end{table}

\subsection{Other Value Distributions}\label{app:valdist}

Figure~\ref{fig:valdist} repeats the \mixshare{50} and
chains experiments at $n=Q$ with bids pulled from different distributions, 
namely uniform and lognormal.
The cache hit rate and TTFT does not
change, since they depend on the tree and not on the bids. \dfsbid's
gain over \dfsuniform\ is positive and significant under every
distribution and grows with the tail, as does \bidsort's gain over
\dfsbid.

\begin{figure}[htbp]
    \centering
    \includegraphics[width=0.7\linewidth]{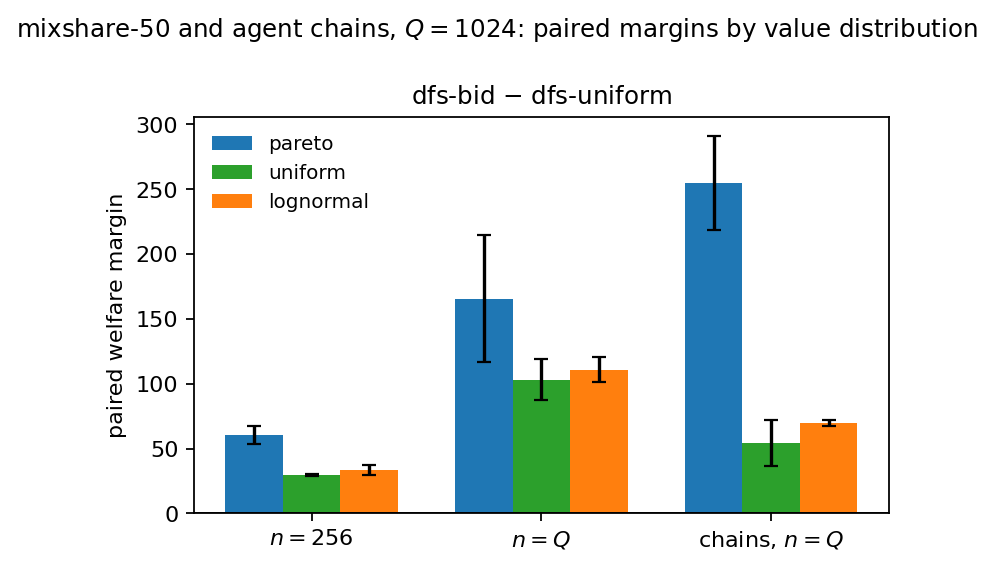}
    \caption{Welfare gain on
    \mixshare{50} and \agent\ under Pareto, uniform on $[1,5]$, and
    lognormal ($\sigma=0.5$) values with the same mean.}
    \label{fig:valdist}
\end{figure}

\subsection{Fast VCG Payments}\label{app:fastvcg}

The experiments compute payments by rerunning the allocation once per
user, which incurs a quadratic run-time. We also implemented the algorithm of
Theorem~\ref{thm:payments}, which re-sorts only the affected children
on a user's root-to-leaf paths. Its payments agree with the naive
computation to $10^{-12}$ on random trees and on our stored instances,
and at $n=1024$ it takes 34\,ms while the naive algorithm takes 370\,ms on average.

\subsection{Autobidding}\label{app:autobid}

We simulate Algorithm~\ref{alg:ogd} over
$T=60$ batches of $n=256$. For stars, each shared-prefix cluster belongs to a distinct user. For chains, each trajectory belongs to a distinct user. Half the users pace toward a target spend
of half their mean truthful payment with step size $\eta=0.1$
normalized by the target; the rest bid truthfully.
Figure~\ref{fig:autobid} shows the usual behavior of dual descent with
a fixed step: spend runs above the target while the multiplier is
small and below it once the multiplier has grown. Welfare under pacing
averages $0.91$ of the DFS optimum at true values on stars and $0.93$
on chains, and the bound of Theorem~\ref{thm:condition} holds in every
batch.

\begin{figure}[htbp]
    \centering
    \includegraphics[width=\linewidth]{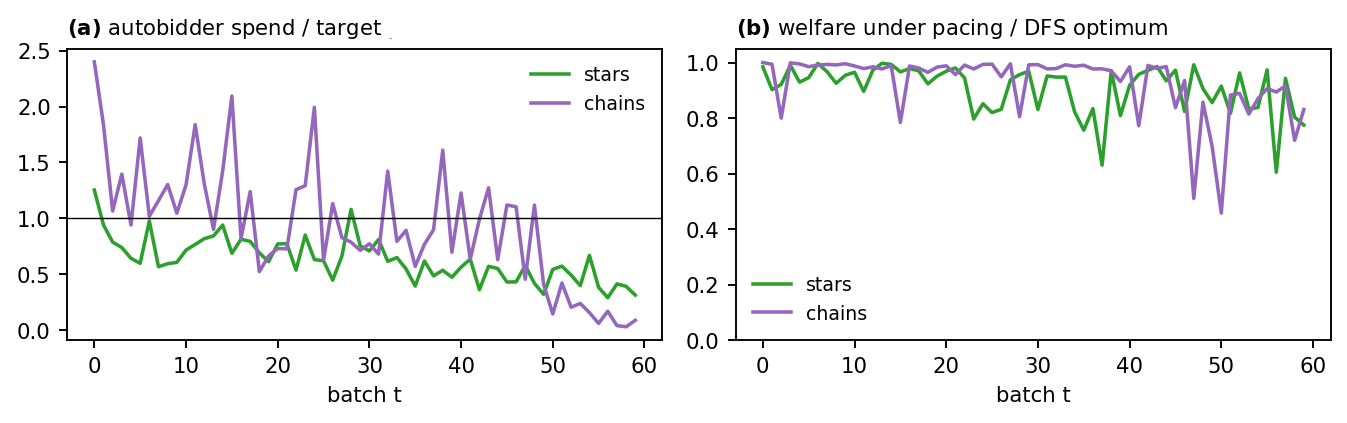}
    \caption{Pacing autobidders over $T=60$ batches of $n=256$, half
    the users pacing and half truthful. (a) Spend relative to the
    per-batch target. (b) Welfare under pacing as a fraction of the
    DFS optimum at true values.}
    \label{fig:autobid}
\end{figure}

\end{document}